\documentclass[11pt, letterpaper, shortlabels]{berkeley}
\title{ArCHer: Training Language Model Agents via Hierarchical Multi-Turn RL}

\usepackage{microtype}
\usepackage{graphicx}
\usepackage{xcolor}
\usepackage{pgf}
\usepackage{booktabs}
\usepackage{subcaption}
\usepackage{booktabs}
\usepackage{xcolor}         
\usepackage{color}         
\usepackage{algorithm}
\usepackage{algorithmicx}
\usepackage{algpseudocode}
\usepackage{mkolar_definitions}
\usepackage{multirow}
\usepackage{subcaption}
\usepackage{resizegather}
\usepackage{hyperref}
\usepackage{color-edits}
\addauthor{yifei}{blue}

\usepackage{algorithm}
\usepackage{amsmath}
\usepackage{amssymb}
\usepackage{mathtools}
\usepackage{amsthm}
\newcommand\numberthis{\addtocounter{equation}{1}\tag{\theequation}}
\ifx\assumption\undefined
\newtheorem{assumption}{Assumption}
\fi

\usepackage[capitalize,noabbrev]{cleveref}

\usepackage[textsize=tiny]{todonotes}
\usepackage{wrapfig}
\title{ArCHer: Training Language Model Agents via Hierarchical Multi-Turn RL}

\usepackage[all]{hypcap}

\usepackage[authoryear, round]{natbib}

\usepackage{hyperref}[citecolor=magenta,linkcolor=magenta]

\hypersetup{
    colorlinks = true,
    citecolor = {magenta},
}

\usepackage{microtype}
\usepackage{graphicx}
\usepackage{booktabs} 
\usepackage{float}

\usepackage{amsmath}
\usepackage{amssymb}
\usepackage{mathtools}
\usepackage{amsthm}
\usepackage{mathrsfs}
\usepackage{nicefrac}
\usepackage{dsfont}
\usepackage{enumitem}
\usepackage{float}

\usepackage{xspace}
\usepackage[capitalize,noabbrev]{cleveref}
\usepackage{subcaption}
\usepackage{wrapfig}
\usepackage{lipsum}
\usepackage{listings}

\usepackage{amsmath}
\usepackage{amssymb}
\usepackage{mathtools}
\usepackage{amsthm}
\usepackage{bbm}

\usepackage{algpseudocode}
\usepackage{setspace}

\usepackage{color}
\definecolor{deepblue}{rgb}{0,0,0.5}
\definecolor{deepred}{rgb}{0.6,0,0}
\definecolor{deepgreen}{rgb}{0,0.5,0}

\newcommand\pythonstyle{\lstset{
basicstyle=\ttfamily\footnotesize,
language=Python,
morekeywords={self, clip, exp, mse_loss, uniform_sample, concatenate, logsumexp},              
keywordstyle=\color{deepblue},
emph={MyClass,__init__},          
emphstyle=\color{deepred},    
stringstyle=\color{deepgreen},
frame=single,                         
showstringspaces=false
}}

\lstnewenvironment{python}[1][]
{
\pythonstyle
\lstset{#1}
}
{}

\newcommand\pythoninline[1]{{\pythonstyle\lstinline!#1!}}

\makeatletter
\def\mathcolor#1#{\@mathcolor{#1}}
\def\@mathcolor#1#2#3{%
  \protect\leavevmode
  \begingroup
    \color#1{#2}#3%
  \endgroup
}
\makeatother

\Crefformat{equation}{#2Eq.\;(#1)#3}

\Crefformat{figure}{#2Figure #1#3}
\Crefformat{assumption}{#2Assumption #1#3}
\Crefname{assumption}{Assumption}{Assumptions}

\usepackage{crossreftools}
\usepackage{dsfont}
\usepackage{nicefrac}

\author[1]{Yifei Zhou}
\author[1]{Andrea Zanette}
\author[1]{Jiayi Pan}
\author[1]{Sergey Levine}
\author[2]{Aviral Kumar}

\affil[1]{University of California, Berkeley}
\affil[2]{Google DeepMind}

\correspondingauthor{yifei\_zhou@berkeley.edu, zanette@berkeley.edu, aviralkumar@google.com}

\begin{abstract}
\textbf{Abstract:} A broad use case of large language models (LLMs) is in goal-directed decision-making tasks (or ``agent'' tasks), where an LLM needs to not just generate probable completions for a given prompt, but rather make intelligent decisions over an extended period of multi-turn interaction to accomplish a task (e.g., when interacting with the web, using software tools, or engaging in customer support). Reinforcement learning (RL) provides a general paradigm to address such agent tasks, but current RL methods for LLMs largely focus on single-turn reward maximization. By construction, single-turn RL methods of today cannot actually train LLMs to intelligently seek and incorporate information over multiple turns, perform credit assignment, or reason about their past actions -- all of which are critical in agent tasks. This raises the question: how can we design effective and efficient multi-turn RL algorithms for LLMs? In this paper, we propose an algorithmic framework for developing multi-turn RL algorithms for fine-tuning LLMs, that preserves the flexibility of existing single-turn RL methods for LLMs (e.g., proximal policy optimization), while accommodating multiple turns, long horizons, and delayed rewards effectively. To do this, our framework adopts a hierarchical RL approach and runs two RL algorithms in parallel: a high-level off-policy RL algorithm that trains a value function to aggregate reward over utterances, and a low-level RL algorithm that utilizes this high-level value function (in place of a reward model used in single-turn RL) to train a token-by-token policy within each utterance or turn. This hierarchical approach prescribed by our framework, \textbf{A}cto\textbf{r}-\textbf{C}ritic Framework with a \textbf{H}i\textbf{e}rarchical Structu\textbf{r}e (\textbf{ArCHer}), can also give rise to a number of other RL approaches. Empirically, we find that ArCHer significantly improves efficiency and performance on multi-turn tasks, attaining sample efficiency of about \textbf{100x} over existing on-policy methods, while also benefitting favorably from scaling up model capacity (upto the 7 billion scale that we could test on in our experiments). Project page can be found in \href{https://yifeizhou02.github.io/archer.io/}{https://yifeizhou02.github.io/archer.io/} and code can be found in \href{https://github.com/YifeiZhou02/ArCHer}{https://github.com/YifeiZhou02/ArCHer}.

\end{abstract}

\begin{document}

\maketitle

\vspace{-0.25cm}
\section{Introduction}
\vspace{-0.2cm}

Owing to their generalist knowledge, large language models (LLMs) have a tremendous potential to address a wide variety of decision-making or ``agent'' problems that can be expressed in text or natural language, from writing code~\citep{intercode, Li2022CompetitionlevelCG, nl2bash}, navigating the web~\citep{Zhou2023WebArenaAR, webshop}, and using tools~\citep{toolformer}, all the way to interacting with humans~\citep{Ghosal2022CICEROAD, chai, humancentric}.
In order to succeed in these domains, an LLM needs to make a \emph{sequence} of intelligent decisions over multiple turns of interaction instead of generating the most \emph{probable} text completion at each step.

Despite these multi-turn agent problems, most methods for eliciting goal-directed behavior from LLMs often rely on myopic objectives that either attempt to mimic successful demonstrations at each step~\citep{agenttuning, Chen2023FireActTL}, or otherwise optimize for single-turn preferences~\citep{touvron2023llama,Ouyang2022TrainingLM, anthropicrlhf}. Policies trained via single-turn approaches often fail to perform effective credit assignment (e.g., they fail to identify good actions that may lead to long-term future performance despite appearing suboptimal at a given step) and do not endow policies with information-seeking behavior, which is important in agent problems (e.g., when dealing with a new tool). Therefore, in this paper, we consider the problem of building multi-turn RL approaches that are able to directly maximize long-term objective of interest (e.g., customer satisfaction at the end of a multi-turn conversation with an LLM assistant), formulated via a scalar reward function.

\begin{figure*}[t]  
    \centering
    \includegraphics[width=0.95\linewidth]{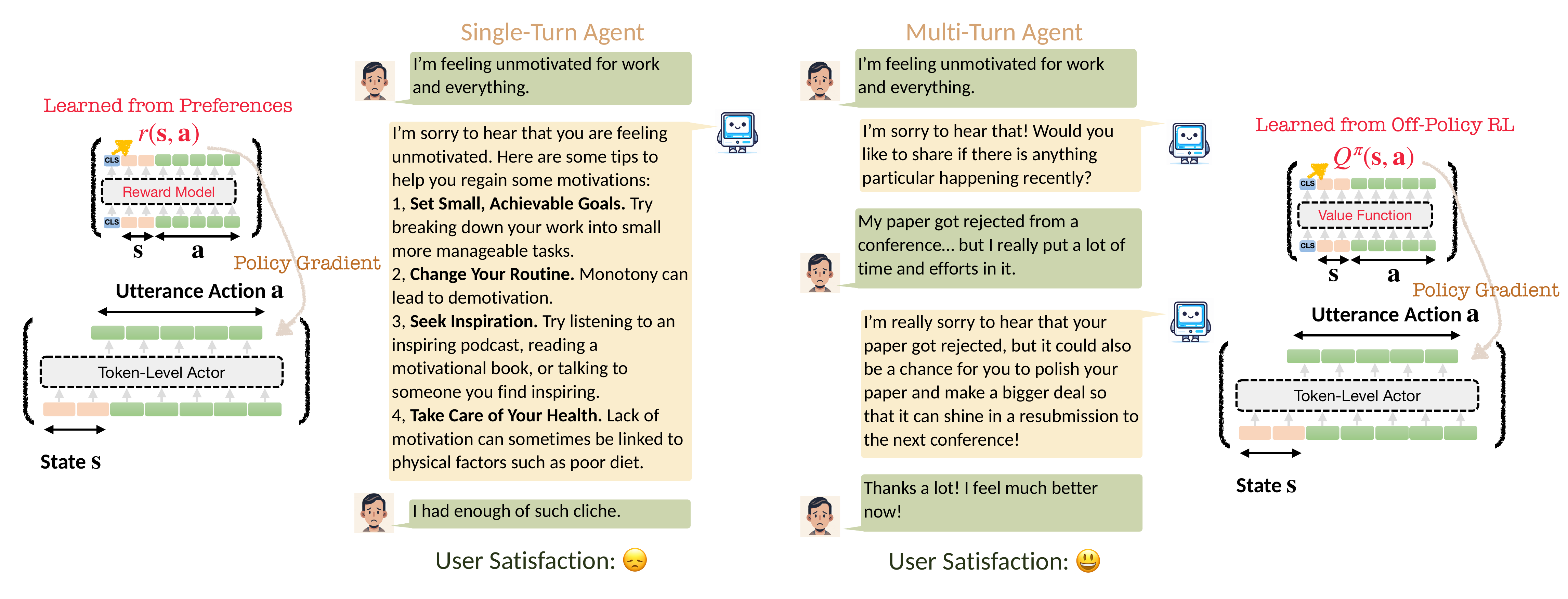}
    \caption{\footnotesize{\textbf{Single-turn RL vs multi-turn RL for LLMs (ours).} In particular, the single-turn RL agent seeks to resolve the request within a single turn and hence ends up providing as much information as possible in its response. On the contrary, the multi-turn RL agent can execute information-gathering actions and address requests in a targeted manner over turns. While current single-turn RL methods for LLMs abstractly use some form of policy gradients computed using a  reward model to train the LLM policy, the method proposed in our paper extends this paradigm to the multi-turn setting by now replacing the reward model with a learned value function, which is trained with off-policy reinforcement learning.}}
    \label{fig:teaser}
\end{figure*}

Unlike single-step RL~\citep{touvron2023llama,Ouyang2022TrainingLM, anthropicrlhf}, training LLMs via multi-turn RL presents a number of unique challenges. First of all, multi-turn RL would require online interaction with external sources such as humans or web servers, which can be slow and expensive. Due to this, on-policy methods such as PPO~\citep{ppo} quickly become impractical owing to their inability to reuse data from past interaction. While off-policy fully offline methods circumvent this problem, these methods present other challenges:  since the number of tokens increase with multiple turns (often substantially so, due to the bias of LLMs towards producing long responses), token-level methods~\citep{ilql,humancentric} that consider individual tokens as actions need to now propagate reward signal over extremely long horizons. This results in extremely slow learning speeds over long horizons for token-level algorithms.
For example, token-level ILQL~\citep{ilql} takes more than 10 days to converge on a task in our experiments, while filtered behavior cloning takes less than a day.
In principle, to address long horizon issues, one can treat the entire utterance for each turn as an action~\citep{chai, gptcritic}, but this comes at the cost of introducing an enormous, variable-size action space, presenting a challenge for off-policy methods based on temporal-difference (TD) learning that require maximization over the action at each time step. This necessitates a multi-turn RL framework that can attain a sweet spot in terms of the aforementioned challenges.

In this paper, we devise a framework for building multi-turn RL algorithms that attains this kind of a sweet spot. Our key insight is that a \textbf{hierarchical approach} for RL with language models that addresses the challenges with both on-policy and off-policy RL approaches as outlined above. Specifically, our framework prescribes an off-policy temporal difference learning method for training an utterance-level value function at the high level, and any on-policy policy gradient algorithm for optimizing the token generation at each turn of the interaction at the low level, treating the high-level value function as the terminal reward for that turn.  Unlike on-policy methods, this allows for sample reuse and faster convergence, while avoiding Bellman backups over individual tokens or maximization over enormous action spaces, as the high-level critic is trained at a coarser time-scale, on tokens produced by the actor. In addition, it also directly inherits implementation details from existing token-level RL algorithms developed for single-turn RL with preferences, for training the policy. This way we are able to obtain the best of both utterance-based and token-based, and off-policy and on-policy approaches for training LLMs.

Our main contribution is a framework for developing hierarchical RL approaches for LLMs, that we call: \textbf{A}cto\textbf{r}-\textbf{C}ritic framework with a \textbf{H}i\textbf{e}rarchical Structu\textbf{r}e (or \textbf{ArCHer} in short). We study several concrete algorithmic instantations derived from the ArCHer framework by conducting experiments on a range of language ``agent'' tasks with active data collection (i.e., the ``online'' setting). We find that algorithms derived from ArCHer are 100x more sample efficient than on-policy methods such as PPO, and converge to a better performance than off-policy methods. Moreover, our methods are easy to build on existing single-turn RL methods and scale to different transformer architectures and more parameters (we show effectiveness of our approach up to the 7 billion scale), directly enabling plug-and-play choices of RL algorithms and models.

\vspace{-0.25cm}
\section{Related Work}
\vspace{-0.25cm}

\textbf{Single-turn reinforcement learning for LLMs.} Most prior works that use RL for LLMs have focused on decision-making problems where the language model must produce a single decision, with no further steps of interaction with an external environment (e.g., the ``single-turn'' preference optimization setting~\citep{casper2023open, christiano2023deep, ziegler2019finetuning}). 
Typical algorithms used for this sort of single-turn RL include policy-gradient methods such as PPO~\citep{Ouyang2022TrainingLM, gpt4, nlpo}, A2C \citep{sparrow}, offline optimization methods such as DPO \citep{dpo}, and filtered supervised learning approaches~\citep{yuan2023rrhf,korbak2023pretraining,gulcehre2023reinforced}. Despite promising results, there remain many important agent problems that cannot be solved in a single-turn setting. Many of these problems require the agent to explicitly take the steps to gather information before making a decision, such as asking for personalized preferences before making a travel plan \citep{hong2023zeroshot} or initially attempting to read the help manual of the shell in a linux terminal, before carrying out the requested tasks~\citep{agentbench}. Single-step approaches cannot learn such nuanced strategies as they attempt to solve the problem within a single step, necessitating multi-turn RL methods for training LLMs.

\textbf{Training language agents without RL.} Motivated by few-shot learning and reasoning abilities of LLMs, prior works also utilize LLMs for sequential decision-making via prompt engineering. For example, ReAct~\citep{yao2023react} and Reflexion~\citep{Shinn2023ReflexionLA} prompt the LLM to ``think'' and analyze past failures before executing the next action. Voyager~\citep{Wang2023VoyagerAO} prompts the LLM agent to develop and refine a curriculum and action library based on environment input. However, without updating the parameters of the LLM, the effectiveness of these methods is inherently limited by the intrinsic capabilities obtained from pre-training \citep{agenttuning, Chen2023FireActTL}. Even state-of-the-art models such as GPT-4 with in-context learning can perform very sub-optimally in out-of-distribution settings~\citep{agentbench,yang2023leandojo} without updating the model. 
To improve over pre-trained capabilities, another line of work finetunes LLMs with successful trajectories (generated manually or by rolling out a strong pre-trained LLM)~\citep{toolformer, agenttuning, Chen2023FireActTL}. However, manual labels and tool call annotations are be expensive to obtain. Moreover, it would be prohibitively expensive for automated approaches to stumble upon successful rollouts will as the task horizon increases~\citep{agentbench, lmrl}. Therefore, in this paper, we sidestep these problems by directly maximizing the objective of interest via RL.
\textbf{Multi-turn RL for LLMs.} While many prior works directly use
off-the-shelf policy-gradient methods, such as PPO \citep{ppo, llarp, webshop} and REINFORCE \citep{suttonpolicy, Williams2004SimpleSG, ranzato2015sequence, wu2018learning, paulus2017deep} to train LMs, these methods can become sample inefficient in multi-step settings that require interaction with an external environment~\citep{chai, gptcritic}. To address such sample complexity issues, off-policy and offline value-based methods learn from existing static data~\citep{ilql, humancentric, chai, gptcritic}. However, existing off-policy methods for multi-turn language tasks either \textbf{(1)} consider a single token as an action (i.e., ``token-level'')~\citep{ilql, humancentric} and  must deal with long horizons, or \textbf{(2)} consider an utterance as a single action, but utilize multiple candidate utterances from a frozen pre-trained LLM for maximization in the Bellman backup~\citep{chai, gptcritic}, reducing the pace of policy improvement, as we also find in our experiments. Our approach will address both of these limitations.

\vspace{-0.2cm}
\section{\textbf{A}cto\textbf{r}-\textbf{C}ritic Framework with a \textbf{H}i\textbf{e}rarchical Structu\textbf{r}e (\textbf{ArCHer})}
\vspace{-0.15cm}

\begin{figure*}[t]  
    \centering
    \includegraphics[width=0.95\linewidth]{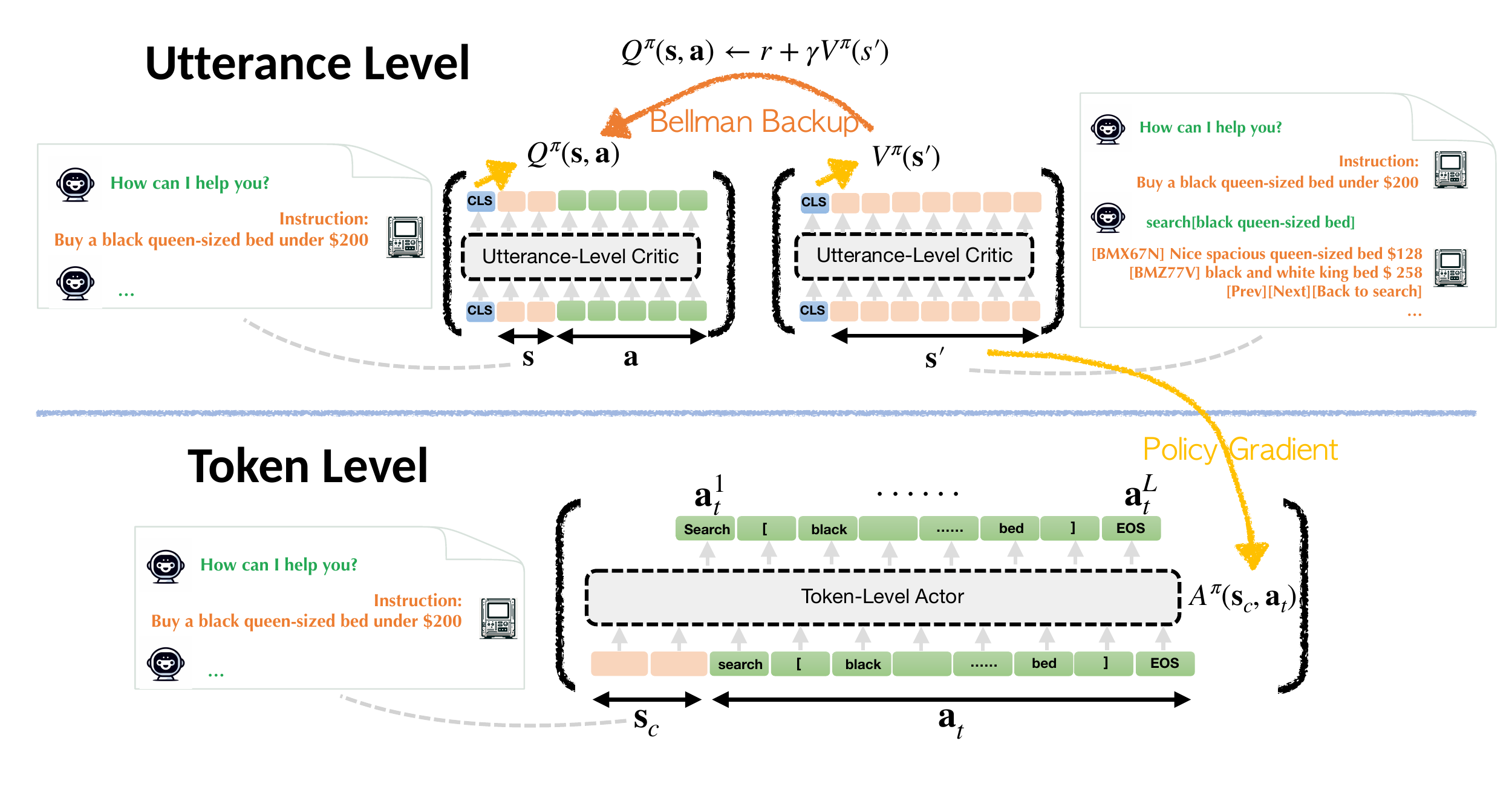}
    \caption{\footnotesize{\textbf{Schematic of the practical instantiation of Actor-Critic Framework with a Hierarchical Structure (ArCHer).} Our algorithm operates both at the utterance level and the token level. At the utterance level, our algorithm learns a Q-function, via Bellman bootstrapping with TD errors. At the token level, the policy is learned by maximizing the advantage function induced by the utterance-level Q-function using a policy gradient approach, where this advantage estimate is provided as a reward at the end of the sequence of tokens appearing within the utterance.}}
    \label{fig:diagram}
    \vspace{-0.5cm}
\end{figure*}
Existing RL methods that consider an individual token as an action suffer from very long horizons with multiple turns. Utterance-level RL methods avoid this challenge, but now they must tractably maximize over a coherent set of tokens within an utterance.
To address these issues, we will develop a class of hierarchical RL methods, ArCHer, that bypass these challenges by running two RL algorithms, in parallel. We start by describing how multi-turn language generation can be posed as a hierarchical Markov decision process (MDP), followed by building RL methods in this hierarchical MDP.

\vspace{-0.2cm}
\subsection{Language Generation as a Hierarchical MDP}
\label{section: hierarchical_mdp}
\vspace{-0.2cm}

In order to derive effective multi-turn RL algorithms, we will now present a novel formulation of language generation as acting in a hierarchical MDP. Our construction defines a \emph{high-level} MDP, and a \emph{low-level} MDP embedded inside the high-level MDP. For both MDPs, states are a variable-length sequence of tokens. For the high-level MDP, an action is defined as a sequence of tokens. An action in the low-level MDP is a single token, such that executing a sequence of actions in the low-level MDP corresponds to a single action in the high-level MDP. 
Formally, each state $s_t$ in the high-level MDP consists of an interaction history between the LLM and the external environment, and each action $a_t$ in this MDP is a variable-length sequence of tokens. The low-level MDP models the generation of the high-level action produced by the agent \emph{within} a single high-level action, where each low-level action $a_t^h$ is an individual token (i.e., the $h$-th token in the $t$-th high-level action). A state in this low-level MDP is obtained by concatenating a high-level state $s_{c}$ consisting of the interaction history until this turn and $a_t^{1:h-1}$, the history of individual action tokens produced within the current turn before step $h$. The next state is obtained by concatenating the current action to this token history.
Figure \ref{fig:diagram} shows a visualization of our notations for the hierarchical MDP.


Policy optimization in the high-level MDP aims to maximize task reward, whereas a policy in the low-level MDP attempts to find a sequence of up to $L$ tokens $a_t^{1:L}$ that maximizes reward equal to the value function of the high-level MDP $Q^\pi(s, a_t^{1:L})$, provided at the end of the low-level rollout. A rollout in the low-level MDP ends as soon as the policy ends up choosing to produce an ``EOS'' token. 

A concrete and natural instantiation of this hierarchical framework in the context of multi-turn interactions is when each turn or an ``utterance'' corresponds to a single time step in the high-level MDP, and each token within a turn is an action in the low-level MDP. In other words, this construction chooses to use the utterance-level MDP (Figure~\ref{fig:diagram}) at the high level and the token-level MDP at the low level. For example, in the context of a web agent, the state would be the interaction of web pages visited so far and a high-level action would be an utterance, e.g., ``search[queen-sized bed, black]''. A candidate reward function would be +1 if the correct item can be bought and 0 otherwise. The dynamics would involve the web engine reacting to an action. Each token of an utterance (e.g., ``search[queen-sized bed, black]'') would be an individual action in the embedded token-level MDP, for example, an action would be individual tokens ``search'', ``['', ``queen''. 


\vspace{-0.25cm}
\subsection{Preliminaries: Reinforcement Learning Definitions} 
\vspace{-0.2cm}

In order to describe the details of our framework, we first provide a few standard RL definitions.
The Q function of a  policy $\pi$ is the expected long-term return obtained by executing a certain action at the current step, followed by executing $\pi$ thereafter: {$Q^\pi(s_h,a_h) = \EE_{\pi} \left[\sum_{t=0}^\infty \gamma^i r(s_{h+t}, a_{h+t})\right]$}.
The value function $V^\pi(s_h)$ is given by taking an expectation of the Q-value, $Q^\pi(s_h, a_h)$, under actions $a_h$ sampled from the policy $\pi$. The advantage $A^\pi(s_h,a_h)$ of a state-action pair is the difference between its Q-value and the value of the state under the policy: $A^\pi(s_h,a_h) = Q^\pi(s_h, a_h) - V^\pi(s_h)$. We will denote the value function in the low-level MDP as $\widetilde{V}(s_c, a_h^{1:i-1})$.

\vspace{-0.2cm}
\subsection{RL Algorithms in the Hierarchical Language MDP}
\label{sec:hierarchical_mdp}
\vspace{-0.2cm}

The proposed hierarchical MDP provides flexibility in designing multi-turn RL algorithms: we could use any choice of RL algorithm for either the high or the low level. That said, note that only the high level requires interaction with a (non-differentiable) environment, while the low level optimizes against the high-level value function, and therefore trains entirely ``in silico,'' without any interaction with an environment. Therefore, the requirements on these methods are different: the high-level algorithm should be highly sample efficient, while the low-level algorithm should be easy to optimize. A particularly convenient choice is to use TD learning at the high level, while using on-policy methods~\citep{Ouyang2022TrainingLM, anthropicrlhf} at the low level.

\vspace{-0.2cm}
\subsection{A Practical Instantiation of ArCHer for Sample-Efficient Online RL}\label{section: practical}
\vspace{-0.15cm}

For deriving a concrete practical algorithm, we will utilize the natural hierarchical MDP induced in multi-turn language interaction: the utterance-level MDP at the high level and the embedded token-level MDP at low level, as discussed at the end of Section~\ref{section: hierarchical_mdp}. In this setting, our approach would train an utterance-level critic with TD backups and a token-level policy with policy gradients. 

\textbf{High-level utterance critic.} Following the practice in prior RL algorithms~\citep{ilql}, we train two LLM models at the high-level, one to represent the utterance-level Q-function $Q^\pi_\theta(s,a)$, and one to represent the utterance-level value function, $V^\pi_\psi(s)$. The Q-model is trained on Bellman targets computed from a delayed copy of the value-model. And the value model, in turn, is trained to approximate the expected value of the Q-model on token sequences (i.e., utterances) obtained by sampling autoregressively from the low-level policy, $\pi_\phi$. 
Due to the off-policy nature of this training process, we store and train on data from all previous online interactions $\Dcal=\{s_i, a_i, r_i, s_i'\}^N_{i=1}$. The objective for training the Q-model and the value-model are formally given by:
\begin{align}
    J_Q(\theta) = \EE_{s, a, r, s' \sim \Dcal}\left[(Q_\theta(s,a) -r - \gamma V_{\bar{\psi}}(s'))^2\right]. &~~~~\text{(Bellman consistency between $Q_\theta$ and $V_{\bar{\psi}}$)} \label{equation: JQ}\\
    J_V(\psi) = \EE_{s \sim \Dcal}\left[\EE_{a \sim \pi_\phi(\cdot|s)}\left[(V_\psi(s) - Q_{\bar{\theta}}(s,a))^2\right]\right]. &~~~~\text{(Train $V_{\psi}$ to approximate $\EE_{a \sim \pi_\phi(\cdot|s)}[Q_{\bar{\theta}}(s, a)])$} \label{equation: JV}
\end{align}
We estimate Equation~\ref{equation: JV} by sampling a batch of $n$ observations $\{s_i\}_{i=1}^n$, followed by auto-regressively sampling token sequences from the actor $\{a_i^L\}_{i=1}^n$. The delayed target models $Q_{\bar{\theta}}$ and $V_{\bar{\psi}}$ are updated towards their current counterparts with Polyak averaging~\citep{sac}.

\textbf{Low-level token actor.} At the low-level, we train the token-level actor $\pi_{\phi}(\cdot|s_c, a_t^{1:h})$ via an on-policy policy gradient approach to find a sequence of tokens that maximizes the prediction of the Q-model. To reduce variance, we use advantage values derived from the Q-model as the terminal reward. This subroutine for training the actor generalizes single-turn RL methods from RLHF, except that the terminal reward is now an estimate of the multi-turn advantage instead of a reward model. Concretely, we update the token-level policy with the policy gradient computed via REINFORCE \citep{Williams2004SimpleSG}:
\vspace{-4mm}
\begin{gather}\label{equation: reinforce}
  J_\phi(\pi) = \mathbb{E}_{s_{c} \sim \mathcal{D}, a_t^{1:L} \sim \pi(\cdot|s_c)} \left[\sum_{i=1}^L A(s_{c}, a_t^{1:L}) \log \pi_{\phi}(a_t^i| s_{c}, a_t^{1:i-1}) \right].\numberthis 
\end{gather}

\vspace{-0.2cm}
\subsection{Other Practical ArCHer Algorithms} 
\label{sec:other.practical}
\vspace{-0.2cm}

The instantiation of ArCHer described in Section \ref{section: practical} is simple and ready-to-use, but the flexibility of the ArCHer framework also enables it to incorporate other components that prior works found to be successful. 
We describe two variants in the two paragraphs to follow. The first improves the sample efficiency while interacting with the external environment and the second enables ArCHer to learn entirely from a dataset of pre-collected experience.

\textbf{Improvements to token-level policy gradient. }\label{para: baseline} In applications where the horizon of the token-level MDP is long (i.e. each utterance has many tokens), despite the use of advantage values (instead of Q-model predictions directly), the REINFORCE estimator corresponding to Equation~\ref{equation: reinforce} can struggle to improve the policy due to a high variance in token-level reward \citep{Schulman2015HighDimensionalCC}.
This variance can be reduced by introducing a baseline value function, $\widetilde{V}_{\eta}(\widetilde{\pi})$ parameterized by $\eta$, in the token-level MDP. For simplicity, we opt to train this token-level baseline via supervised regression onto Monte-Carlo return estimates (with a discount factor of $1.0$) in the token-level MDP as shown in Equation~\ref{equation: baseline}, though Bellman backups can also be employed in the low-level MDP to estimate it:
\begin{gather}\label{equation: baseline}
    J_{\eta}(\widetilde{V}) = \EE_{s_{c} \sim \Dcal, a_t^{1:L} \sim \pi(\cdot|s_c)} \left[\sum_{i=1}^L \left(A(s_{c}, a_t^{1:L}) - \widetilde{V}_{\eta}(s_c, a_t^{1:i-1})\right)^2\right]. \numberthis 
\end{gather}
Incorporating this token-level baseline, the new objective for the actor given by Equation \ref{equation: reinforce with baseline}:
\begin{align*}
    J_\phi(\pi) &= \EE_{s_{c} \sim \Dcal, a_t^{1:L} \sim \pi(\cdot|s_c)} \left[ \sum_{i=1}^L  \left(A(s_{c}, a_t^{1:L}) - \widetilde{V}_{\eta}(s_c, a_t^{1:i-1}) \right) \cdot \log \pi_{\phi}(a_t^i| s_{c}, a_t^{1:i-1}) \right]. \numberthis \label{equation: reinforce with baseline}
\end{align*}

\textbf{Offline RL training with ArCHer.} ArCHer can also learn from a dataset of pre-collected experience without any online interaction.
A distinguishing aspect of the offline setting is that improving the policy normally results in selecting out-of-distribution actions~\citep{kumar2019stabilizing}, whose $Q$ values are difficult to estimate accurately given only the offline dataset. As a result, directly optimizing the $Q$ values, as in the online setting, often results in severe overestimation, divergence, and poor policy performance. This suggests that we need to adopt different objective functions in this offline setting. One concrete instantiation is to utilize the implicit Q-learning (IQL)~\citep{iql} algorithm for obtaining backup targets for the utterance-level critic restricted to in-support actions, and the AWR~\citep{awr} algorithm for imposing a penalty on deviating far away from the data on the actor. While our offline RL experiments utilize these design choices, one could also utilize other techniques such as explicitly regularizing the critic's predictions~\citep{cql} or imposing a behavioral cloning loss on the actor~\citep{fujimoto2021minimalist}.

The IQL loss aims to derive a version of the TD error that aims to inherit characteristics of the Bellman optimality operator but without performing an explicit maximization over the actions, by instead regressing Q-functions towards a higher expectile of possible target values at the next state. 
For a given expectile parameter $\tau \in [0.5, 1)$, the IQL loss is given by the following:
\begin{align}
    \!\!\!\!J^{\text{IQL}}_\psi(V) = \EE_{s \sim \Dcal}[\EE_{a \sim \pi_\phi(\cdot|s)}[L^\tau_2(V_\psi(s) - Q_{\bar{\theta}}(s,a))]], \label{equation: iql}
\end{align}
where \mbox{$L^\tau_2(u) = |\tau - \mathbf{1}\{u < 0\}|u^2$}.
See the paper \cite{iql} for more details.

\def\invtemp{\ensuremath{\beta}}
\def\expectile{\ensuremath{\tau}}
\def\archer{ArCHer}

The policy extracted by AWR~\citep{awr} trades off between
searching for high-return policies and imitating the  policy that generated the dataset, by minimizing the loss
\begin{gather}\label{equation: awr}
        J_\phi(\pi) = -\EE_{(s_{c}, a_t^{1:L}) \sim \Dcal} \left[\exp\left({\invtemp \cdot A(s_{c}, a_t^{1:L})}\right) \cdot  \sum_{i=1}^L \log \pi_\phi(a_t^i| s_{c}, a_t^{1:i-1}) \right].
    \numberthis
\end{gather}
The tradeoff is controlled by
a positive, user-defined scalar value \invtemp.
Low values for \invtemp{} encourage imitating the policy that generated the dataset. 
Large values of \invtemp{} lead to more aggressive maximization of rewards, but potentially at the cost of stability~\citep{awr}.

\subsection{Framework Summary and Practical Implementation Details}
\textbf{Pseudocode.} The algorithms derived from the ArCHer framework so far are summarized in Algorithm \ref{alg:archer_detail}. These algorithms can operate in either offline or online mode (Line 4), and can utilize a variety of objectives for training the utterance-level Q- and V-models (Lines 9-15) as well the token-level policy (Line 20). Optionally, a token-level baseline value function may also be utilized (Line 17). 

\definecolor{darkgreen}{rgb}{0, 0.5, 0}
\begin{algorithm}[t]
\caption{ArCHer: Practical Framework}
\label{alg:archer_detail}
\begin{algorithmic}[1]
\State Initialize parameters $\phi, \psi, \theta, \bar{\theta}, \text{(Optionally) } \eta$
\State Initialize replay buffer $\mathcal{D}$ (optionally from an offline dataset).
\For{each iteration}
\State \textcolor{darkgreen}{\#\# Data Collection.}\Comment[only online mode] 
\For{each environment step}
\State Execute $a_t \sim \pi_\phi(\cdot|s_t)$ , obtain the next state $s_{t+1}$, add to buffer $\Dcal$.
\EndFor
\For{each critic step}
\State \textcolor{darkgreen}{\#\# Update utterance-level Q and V functions by target function bootstrapping.}
\State $\theta \leftarrow \theta - \nabla J_\theta(Q)$ \Comment{Equation~\ref{equation: JQ}}
\State $\psi \leftarrow \psi - \nabla J_\psi(V)$ \Comment{Equation~\ref{equation: JV} or ~\ref{equation: iql}}
\State \textcolor{darkgreen}{\#\# Update target Q and V functions.}
\State $\bar{\theta} \leftarrow (1 - \tau)\bar{\theta} + \tau\theta$
\State $\bar{\psi} \leftarrow (1 - \tau)\bar{\psi} + \tau\psi$
\EndFor
\State \textcolor{darkgreen}{\#\# Update token-level baseline by MC regression.}
\For{each baseline step}
\State $\eta \leftarrow \eta - \nabla J_{\eta}(\widetilde{V})$ \Comment{(Optionally), Equation~\ref{equation: baseline}}
\EndFor
\State \textcolor{darkgreen}{\#\# Update token-level actor with utterance-level critic.}
\For{each actor step}
\State $\phi \leftarrow \phi - \nabla J_\phi(\pi)$ \Comment{Equation~\ref{equation: reinforce}, \ref{equation: reinforce with baseline}, or \ref{equation: awr}}
\EndFor
\EndFor
\end{algorithmic}
\end{algorithm}

\textbf{Implementation details.} \label{para:implementation_details}In our main experiments, we use a GPT-2~\citep{gpt2} architecture for parameterizing the policy, 
and a RoBERTa-base model~\citep{roberta} with a linear layer on top of the embeddings corresponding to the ``[CLS]'' token for obtaining the critic's predictions. To address the issue of overestimation of Q-values, we also employ the double Q-learning trick \citep{doubleq} and train two copies of Q- and V-models, $\left\{Q_1, V_1\right\}$ and $\{Q_2, V_2\}$, independently. The advantage value is calculated by using a minimum over $Q_1, Q_2$ and $V_1, V_2$.

To save computation and memory costs, $Q_1, Q_2, V_1, V_2$ share the same language model encoder backbone with separate MLP heads. The parameters of the token-level actor are independent from the critic. When utilized, the token-level value baseline is parameterized by a separate GPT2 architecture with a MLP layer on top of the hidden states for each token. Additional details and hyperparameters for our approach are provided in Appendix \ref{app: hyperparameters}. 
\vspace{-0.2cm}
\section{Theoretical Analysis}
\vspace{-0.2cm}

We will present empirical results showing the effectiveness of ArCHer in Section \ref{section: Experiments}, but in this section we will first highlight some theoretical and conceptual benefits of our hierarchical design. An important difference between ArCHer and prior token-level RL algorithms such as ILQL~\citep{ilql} is the high-level critic. Thus, we aim to understand the impact of estimation errors in this high-level critic on the token-level policy in contrast with the impact of estimation errors in a token-level critic. While our proof techniques can be easily applied to general high-level and low-level MDPs, we focus on analyzing the specific contrast of utterance and token-level critics to be concrete.

\textbf{Conditions for convergence.} To start, we show that a hierarchical RL design prescribed by ArCHer requires substantially weaker conditions for algorithm convergence compared to off-policy token-level methods. These conditions pertain to \textbf{(1)} the capacity of the function class representing the critic (i.e., Bellman completeness~\citep{song2023hybrid, zhou2023offline, zanette2023realizability, xie2021bellman}), and \textbf{(2)} the coverage of off-policy data as measured by the density ratio~\citep{zhan2022offline, foster2021offline}, following the practice standard in RL theory. For \textbf{(1)}, we show in Lemma~\ref{lem: bellman_complete} that satisfying the Bellman completeness~\citep{song2023hybrid, zhou2023offline, zanette2023realizability, xie2021bellman} condition
imposes weaker requirements on the function class used to model the critic at the utterance-level (as is the case with ArCHer) as opposed to the token level. Intuitively, this is because a function class representing the token-level critic must exhibit flexibility to realize arbitrary functions at the next token, which would require higher capacity compared to a utterance-level critic that only needs to be able to realize arbitrary functions at the coarser time-scale of utterances. 

For \textbf{(2)}, we show in Lemma~\ref{lem: density_ratio} that the density ratio condition~\citep{zhan2022offline, foster2021offline} imposes identical requirements on the coverage of the offline data for token-level and utterance-level critic, despite a larger space of possible utterances (i.e., actions in the higher-level MDP). Intuitively, this is because the a given offline dataset induces the same trajectory distribution at both the utterance and token levels. In other words, if a trajectory is covered by the offline data at the utterance level, it is also covered by the offline data at the token level, and vice versa.

\newcommand{\az}[1]{{[\color{blue} AZ: #1]}}
\textbf{Statistical error analysis with finite samples.} With the convergence conditions discussed above, we are able to establish an analysis of the statistical error incurred in estimating the advantage estimates using the utterance-level and token-level critics. Intuitively, this means that an utterance-level critic provides a much more correct signal for improving the policy. We state the conclusions of the theorem below, and refer interested readers to Appendix~\ref{proofs} for formal definitions and proof.
\begin{theorem}[Main Theorem; Informal]\label{lem:main_theorem}
    For an utterance-level MDP with discount factor $\gamma$, where $L$ is the maximum length of each utterance, suppose utterance-level Assumption \ref{def:inherent_BE_error} and \ref{def:offline_coverage} holds, let $f$ be the final Q-function returned by fitted policy evaluation formalized in Algorithm \ref{alg:FPE} at the utterance level, $f$ yields a suboptimality gap of
    \begin{align*}
        &\EE_{s,a \sim d^\pi}\left[\left((\bar{f}(s,a) - \EE_{a' \sim \pi(\cdot|s)}[\bar{f}(s,a)]) - A^\pi(s,a) \right)^2 \right] \le \frac{1}{\gamma L^{1/2}}\Ocal \left(\frac{1}{(1-\gamma)(1 - \gamma^{1/L})L^{1/2}}\left(\epsilon_{stat} + \sqrt{\epsilon_{stat}}\right) \right).
    \end{align*}
    For an equivalent token-level MDP with discount factor $\gamma^{1/L}$, suppose token-level Assumption \ref{def:inherent_BE_error} and \ref{def:offline_coverage} holds, let $f$ be the final Q function returned by Fitted Policy Evaluation formalized in Algorithm \ref{alg:FPE} at the token level, $f$ yields a suboptimality gap of
    \begin{align*}
        &\EE_{s,a \sim \widetilde{d}^\pi}\left[\left((\bar{f}(s,a) - \EE_{a' \sim \pi(\cdot|s)}[\bar{f}(s,a)]) - \widetilde{A}^\pi(s,a)\right)^2 \right] \le \Ocal \left(\frac{1}{(1-\gamma)(1 - \gamma^{1/L})L^{1/2}}\left(\epsilon_{stat} + \sqrt{\epsilon_{stat}}\right)\right),
    \end{align*}
    where $\epsilon_{stat}$ is the statistical error, proportional to $N^{-1/2}$ ($N$ is the number of utterance-level transitions). 

\end{theorem}
Informally, Theorem \ref{lem:main_theorem} shows that the error in estimating advantages using the token-level critic is $\gamma \sqrt{L}$ larger than the the utterance-level critic (in the worst case), where $L$ is the maximum number of tokens in each utterance, due to error accumulation. In practice, a common choice for $\gamma$ is greater than $0.95$ while $L$ can be as large as $64$ tokens, resulting in $\gamma L^{1/2} >> 1$. Therefore, we have not only shown  that a hierarchical design requires weaker conditions for convergence, but it also enjoys improved guarantees on the statistical error, resulting in a more accurate estimate of policy gradient for improving the policy in the worst case.

\vspace{-0.2cm}
\section{Experiments}\label{section: Experiments}
\vspace{-0.2cm}

The goal of our experiments is to evaluate the efficacy of hierarchical RL algorithms derived from ArCHer. Specifically, we aim to answer the following questions: \textbf{(1)} Is ArCHer able to achieve better sample complexity and performance than prior on-policy and off-policy RL methods for LLMs? \textbf{(2)} Does the TD-learning design for the utterance-level critic enable an effective use of off-policy data? \textbf{(3)} How does the performance of ArCHer scale with larger base models (such as Mistral 7B~\citep{jiang2023mistral})? \textbf{(4)} How do different practical algorithms derived from our ArCHer framework compare? To answer these questions, we will present an extensive empirical evaluation of ArCHer and several prior methods on a suite of environments encompassing natural language games, navigation problems posed in natural language, and interaction with the web.

\vspace{-0.25cm}
\subsection{Tasks and Environments}
\vspace{-0.2cm}

To stress-test the efficacy of ArCHer, we need environments and task setups that satisfy several desiderata. First, the chosen tasks must require strategic multi-step planning and reasoning under delayed rewards, and cannot be solved in one turn. We also want these tasks to require LLMs to generate coherent natural language and keep the task realistics.

Next, we want these tasks to be solvable by models of upto 7 billion parameter scale, which corresponds to the upper end of our computational budget. Finally, the chosen tasks should support fast and reliable evaluations, for reproducibility and benchmarking. Most existing LLM agent tasks such as those which require interacting with terminals, operating systems, and databases~\citep{intercode, agentbench} require larger base models (typically larger than 7B) for obtaining non-trivial success rates and can often be solved in a single step. This makes these tasks unfavorable for fast iteration within our compute budget. Other dialogue and tutoring  tasks~\citep{hong2023zeroshot, chai, convlab2, convlab, multiwoz} require either costly user studies or evaluate using metrics that do not directly represent task performance, making them unfavorable for stress-testing our approach. Therefore, we utilize a different set of tasks for our evaluations.
 
Concretely, we consider the following tasks: \textbf{(1)} \textbf{Detective Game}~\citep{textscifi}, an interactive text fiction game where the agent must generate a sequence of natural language actions (e.g.,``take paper'', ``eat apple'') based on environment feedback. A reward is given if the agent reaches some milestones towards finishing the game, where the end goal is to successfully find the murderer in a murder mystery; \textbf{(2)}\textbf{Twenty Questions}~\citep{lmrl}, a dialogue task where the agent plays the role of a guesser trying to guess a hidden word from a list of 157 words within twenty yes/no questions. The oracle answers the questions with ``Yes.'', ``No.'', or ``Invalid Question.'' The oracle is simulated with a ``flan-t5-small''~\citep{chung2022scaling} model trained with supervised fine-tuning on the dataset provided by \citet{lmrl}. Upon guessing the correct word, the agent receives a reward of 0 and the environment terminates. Otherwise, a reward of -1 is provided at each time step. We also study a variation of this task with a list of only 10 possible underlying words, that we call \textbf{Twenty Questions Subset}. This variant challenges the algorithms to tailor a very specific strategy when there is a shortcut in the task;
\textbf{(3)} \textbf{Guess My City}~\citep{lmrl}, a similar multi-turn task where the agent attempts to guess the name of a hidden city from a list of 100 cities within twenty questions. A crucial difference between the Guess My City task and the Twenty Questions task is that the guesser is now allowed to ask any question and can observe free-form responses (which are not necessarily ``Yes'' or ``No''); \textbf{(4)} \textbf{WebShop}~\citep{webshop}, a tool-use task where the agent is instructed to buy an item from a shopping server. A dense reward between 0 and 1 is provided based on the similarity of the item purchased and the item requested. See Appendix~\ref{app: dataset_details} for more details. These tasks require planning over long horizons, allow non-trivial success rates within the 7 billion parameter scale, and come equipped with reproducible and task-directed evaluation metrics.

\vspace{-0.2cm}
\subsection{Comparisons and Baseline Approaches}
\label{sec:online_comparisons}
\vspace{-0.2cm}

We compare our method to a number of alternative RL approaches.
For token-level methods, we consider token-level \textbf{PPO} \citep{ppo} due to its state-of-the-art performance with LLMs \citep{Ouyang2022TrainingLM, Ramamurthy2022IsRL}. For each iteration, \textbf{PPO} collects new on-policy data by rolling out the current actor and uses these data to estimate the policy gradient. Perhaps most importantly, data collected by previous policies is simply discarded for later updates. 
We use the existing PPO implementation by \citet{lmrl}. 
We also implemented a \textbf{token-level DQN}~\citep{dqn} method, but were unable to get it to attain non-zero task performance. Neither did we find any prior work evaluating this method, and hence we omit it from our results.

We also considered a non-RL approach based on filtered behavioral cloning (BC), denoted as \textbf{Filtered BC}. As prior work~\citep{lmrl,ilql} shows, perhaps surprisingly, this simple baseline often attains competitive or better performance than RL approaches for LLMs, implying that outperforming filtered BC is a hallmark of proper functioning of the ``RL component'' in an approach. Our implementation of filtered BC maintains a fixed size buffer of recent rollouts and trains the actor with an imitation learning loss on the top 10\% rollouts, as identified based on the totaltask reward. 

For utterance-level methods, we compare with a state-of-the-art
utterance-level RL method, \textbf{CHAI}~\citep{chai}. CHAI was designed for offline RL specifically. To extend it to learn from online rollouts, we simply replace the pessimistic loss function (i.e., a conservative Q-learning~\citep{cql} loss) in this approach with a standard TD-learning loss function on data sampled from the off-policy replay buffer, identical to the one used in ArCHer. That said, the key difference is that CHAI utilizes a frozen actor obtained by behavioral cloning (or supervised fine-tuning) on the replay buffer, whereas ArCHer optimizes the actor as well. Each time when an action needs to be sampled, $k$ utterances are sampled from the frozen actor. The utterance-level critic in CHAI ranks these $k$ utterances and chooses the utterance with the highest Q value. A larger value of $k$ would likely lead to better performance, but is also computationally expensive for this method. To obtain a sweet spot, we instantiate CHAI with $k=5$ to effectively balance computational overhead and performance (\textbf{note that $k=5$ already results in a runtime of about 4 times longer than our approach for this prior method}; going beyond would be prohibitive for our computational resources). 
Model architectures, learning rates, and other algorithm-agnostic details are kept identical between all prior methods and ArCHer (see Section \ref{para:implementation_details}). 

Finally, akin to single-turn RL fine-tuning of LLMs, we initialize the token-level policy for all methods with a policy checkpoint obtained by running supervised instruction tuning on sub-optimal data for the task (see Appendix \ref{app: dataset_details} for how this sub-optimal data is generated). Uniformly across all methods, this initialization enables effective exploration at the beginning of RL fine-tuning.

\begin{figure*}[h!]  
    \centering
    \includegraphics[width=1.\linewidth]{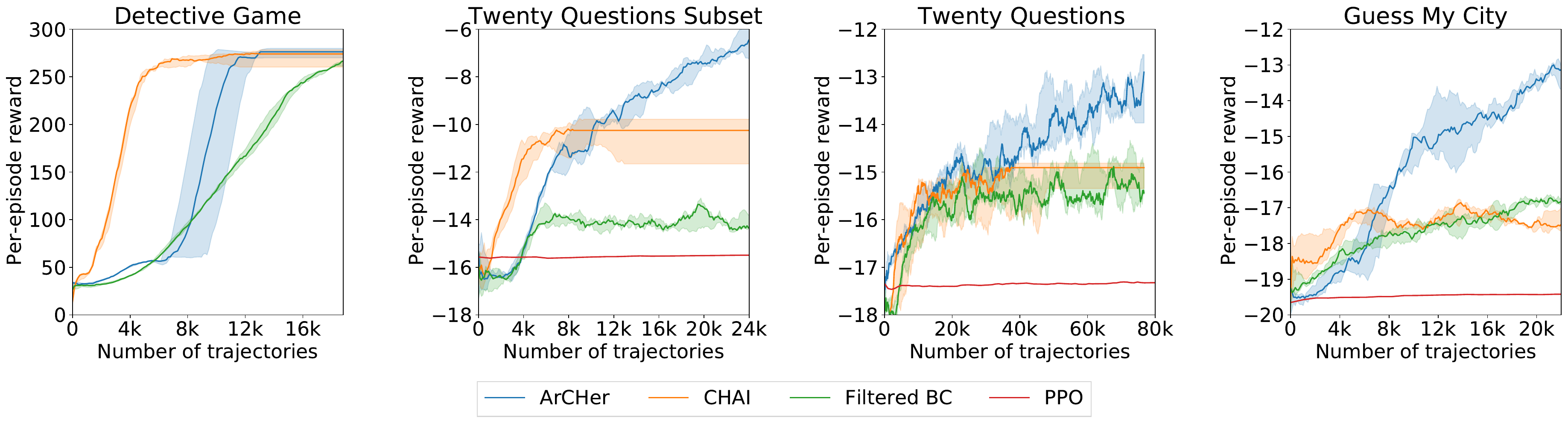}
    \caption{\footnotesize{\textbf{Online RL results} comparing ArCHer and other approaches on four tasks. We plot the median performance of each method across three seeds. Observe that ArCHer steadily improves the policy, outperforming all other methods on three tasks and matching the best prior approach on the simple Detective Game task. While PPO appears to not be learning, by zooming into the learning curve in Figure~\ref{fig: ppo}, we find that PPO still gradually improves but at a very slow speed.}}
    \label{fig: main_online}
    \vspace*{-10pt}
\end{figure*}

\begin{figure}
  \centering
  \includegraphics[width=0.4
\textwidth]{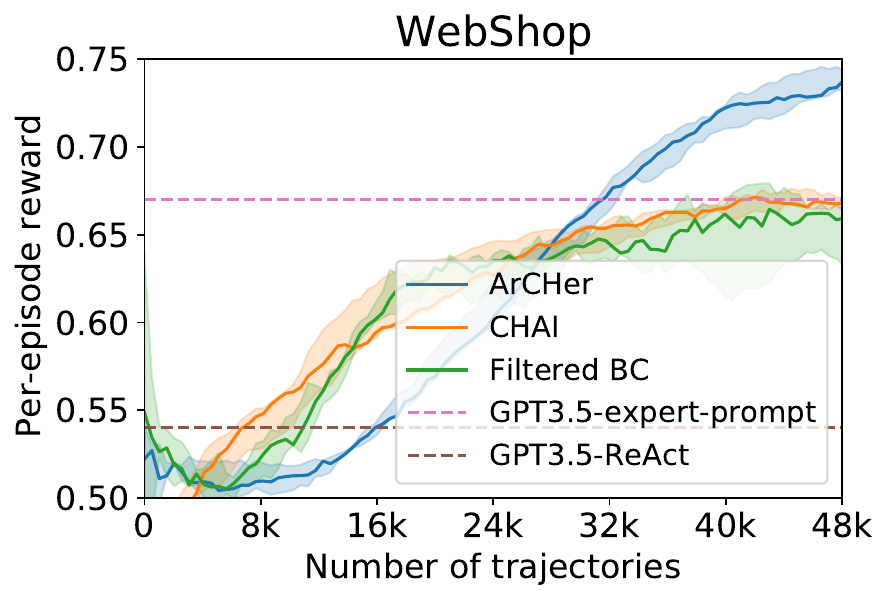}
  \caption{\footnotesize{\textbf{Webshop results.} Observe that fine-tuning a GPT2 base model with ArCHer outperforms prior approaches, filtered BC and CHAI, and is the only approach to outperform GPT 3.5 equipped with several effective prompting strategies.}}\label{fig: webshop} 
  \vspace{-0.2cm}
\end{figure}

\vspace{-0.2cm}
\subsection{Results: Sample Efficiency in the Online Setting}
\vspace{-0.2cm}

Figure \ref{fig: main_online} and \ref{fig: webshop} show the comparison between ArCHer with other methods across the five tasks. We also provide some example rollouts of ArCHer for each environment in Appendix \ref{app: hacking}. 
Overall, we found that ArCHer converges to a much better performance than all other prior methods on the four harder tasks that require identifying hidden information or present diverse initial states (i.e., Twenty Questions Subset, Twenty Questions, Guess My City, and WebShop).
In fact, on WebShop, online RL training of GPT2 base model via ArCHer outperforms several effective prompting strategies (i.e., an expert-written prompt and ReAct~\citep{yao2023react}) applied on top of GPT-3.5, a strong LLM. 

First of all, we found that token-level PPO fails to achieve performance competitive with all other off-policy methods using the same amount of data. This is perhaps unsurprising, as PPO is an on-policy method, and therefore, cannot effectively reuse samples from past iterations. In particular, on the \textbf{Twenty Questions} task,
we observed that PPO could only stably improve when provided with at least 1024 on-policy rollouts for each gradient step, likely because of high gradient variance. This observation corroborates the finding of \citet{lmrl}, suggesting that online PPO is less practical for this task. Quantitatively, we find that while it takes more than 100k samples for PPO to attain an average return just higher than -17 (see Figure~\ref{fig: ppo} in the appendix), ArCHer attains this reward value with fewer than 1000 samples, implying at least a \textbf{100x boost} in sample efficiency.

While filtered BC generally converges very quickly, the resulting policy often performs suboptimally and does not improve with more data in Figure~\ref{fig: main_online}. On the other hand, ArCHer enjoys steady policy improvement as it collects more samples. Finally, we observed that while CHAI improved at a faster than ArCHer initially, it often converged to a worse final performance. We suspect this is because the critic in CHAI is directly used to rerank samples from a frozen behavior policy, which only enables a narrow margin for policy improvement. On the contrary, ArCHer needs an initial learning phase to reduce critic estimation error, after which it can improve steadily. 
\vspace{-0.2cm}
\subsection{Ablation Study: Importance of Off-Policy Data}
\vspace{-0.2cm}

In Figure \ref{fig: token_baseline} (b), we investigate the importance of off-policy data by varying the size of replay buffer on the \textbf{Guess My City} task. A smaller replay buffer means that updates rely on repeatedly sampling on-policy data. In our experiments, we observed that using a replay buffer containing only the most recent 48 rollouts resulted in unstable learning, likely due to overfitting on limited data, which has been observed in standard RL problems outside of LLMs~\citep{nikishin2022primacy}. On the other hand, larger buffers are more stable. However, increasing the size of the buffer beyond a certain point is benign, resulting in no meaningful changes to performance. Overall, this means that making use of off-policy data can improve the stability and performance of practical methods.

\begin{figure}[ht]
    \centering
    \includegraphics[width=0.99\linewidth]{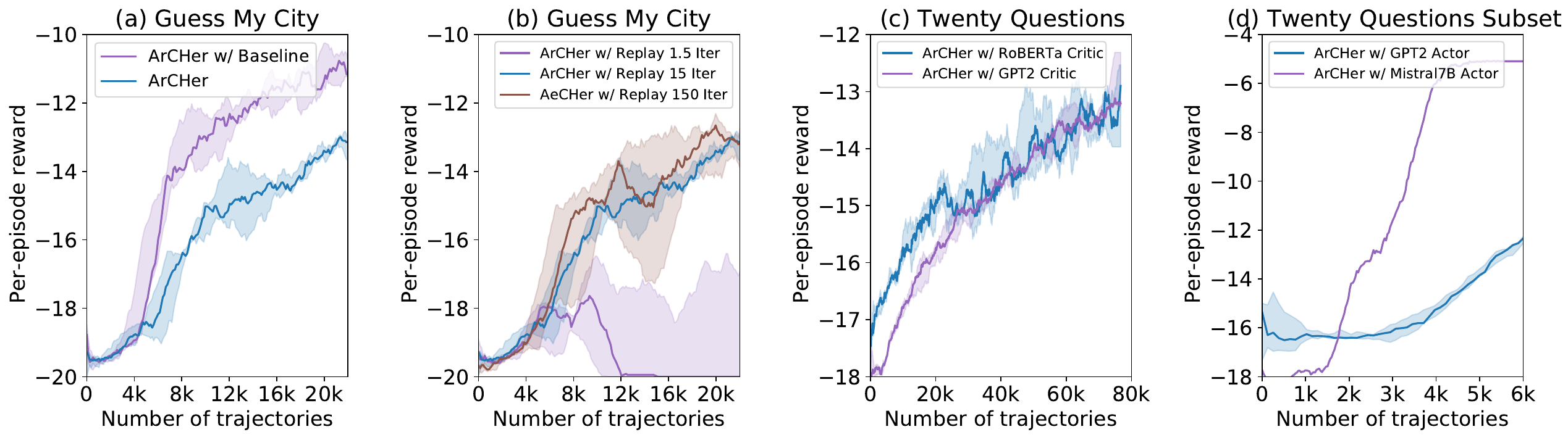}
    \vspace{0.05cm}
    \caption{\footnotesize{\textbf{(a)} Ablation study of the token-level baseline on Guess My City. \textbf{(b)} Ablation study of the importance of off-policy data by varying the size of the replay buffer on Guess My City. \textbf{(c)} Ablation study of changing the base model for critic from encoder-only RoBERTa to autoregressive decoder-only GPT2 on Twenty Questions. \textbf{(d)} Ablation study of scaling the base model for the actor from GPT2 to Mistral7B on Twenty Questions Subset.}}
    \label{fig: token_baseline}
    \vspace{-0.2cm}
\end{figure}

\vspace{-0.2cm}
\subsection{Ablation Study: Alternate Base Models for the High-Level Critic in ArCHer}
\vspace{-0.2cm}

In Figure \ref{fig: token_baseline} (c), we carried out an ablation of changing the architecture for the critic model from an encoder-only RoBERTa~\citep{roberta} to an autoregressive decoder-only GPT2~\citep{gpt2}, where we took the embedding of the last ``[EOS]'' token as the embedding of the utterances. Observe that although ArCHer w/ RoBERTa critic learns a bit faster in the beginning, learning curves for both of these critic models behave identically past a certain number of initial samples. Therefore, ArCHer can also use decoder-only transformer models, with no loss in performance.

\vspace{-0.2cm}
\subsection{Ablation Study: Scaling the Base Model from 100M to 7B Parameters}
\vspace{-0.2cm}
In Figure \ref{fig: token_baseline} (d), we replaced the 100 million parameter GPT-2 model used to represent the token-level actor in ArCHer with a 7 billion parameter Mistral model~\citep{jiang2023mistral}. When using this 7B model, we did not need to apply supervised fine-tuning since the open-source checkpoint already attained non-trivial rewards on \textbf{Twenty Questions Subset} when evaluated zero-shot. Observe in Figure~\ref{fig: token_baseline} (d), that ArCHer with this Mistral7B actor learns to solve the task much faster than ArCHer with a GPT2 Actor. This indicates that our ArCHer framework can scale well with LLMs with more parameters. More broadly, due to similarities between the token-level actor update in ArCHer and single-turn RL fine-tuning for LLMs in RLHF, we would expect performance to exhibit similar benefits from scaling the model size for the policy~\citep{gao2023scaling}. This finding corroborates this hypothesis.

\vspace{-0.2cm}
\subsection{Alternate Practical Algorithms Derived from ArCHer}
\vspace{-0.2cm}

\textbf{Online ArCHer with improved policy gradient estimators.} In Figure \ref{fig: token_baseline} (a), we compare the performance of ArCHer with and without a token-level baseline on the \textbf{Guess My City} task (Equation~\ref{equation: reinforce with baseline}). This task requires the utterances of the agent in each turn to be longer and more diverse than other tasks. Observe that incorporating this token-level baseline in ArCHer outperforms standard ArCHer by a large margin, supporting our hypothesis that the introduction of the token-level baseline can effectively reduce the variance of the vanilla policy gradient while updating the token-level actor (especially when the utterances are long and diverse). That said, this improvement requires paying an extra computational overhead associated with training $\tilde{V}_\eta$, which might not be necessary when each utterance is short. Overall, this study illustrates one of the central benefits of our hierarchical design: we can choose the best method for the higher and lower level based on their distinct requirements.

\textbf{Offline ArCHer with IQL and AWR.} We now present a preliminary study of ArCHer in the offline setting, when learning from a static dataset from past environment interactions. Due to computational constraints, we were not able to perform extensive comparisons with the state-of-the-art algorithms; rather, we investigated the effect of several design choices in order to investigate the effect of various design choices in the offline setting, including IQL and AWR losses described in \cref{sec:other.practical}. We also incorporate a baseline, \textbf{BC}, which performs (unfiltered) imitation learning on the offline dataset. Finally, we also ran filtered BC, which only imitates the best trajectories in the offline dataset.

\begin{table}[t]
\begin{center}
\begin{tabular}{ lr}
\toprule
 & Twenty Questions (Return) \\
 \midrule
ArCHer (IQL + AWR)            & $\mathbf{-14.1} $  \\
ArCHer (IQL + REINFORCE)       & $-20$  \\
ArCHer (IQL + REINFORCE + BC)  & $-15.3$\\
ArCHer (SARSA + AWR)  & $-14.5$ \\
\midrule
Filtered BC     & $-15.4 $  \\
BC              & $-16.8 $\\
 \bottomrule
\end{tabular}
\end{center}
\vspace{-0.4cm}
\caption{\footnotesize{\textbf{Variants of ArCHer in the offline RL setting.} The performance of different approaches is evaluated by running 1280 trajectories across 5 random seeds for ArCHer and the Filtered BC approach. Observe that handling out-of-distribution actions by instantiating ArCHer with the IQL and AWR objectives works best in the offline setting.}}
\label{tab:offline}
\vspace{-0.2cm}
\end{table}

In Table~\ref{tab:offline}, we also evaluate several other design choices in the offline setting. Directly borrowing the REINFORCE objective from the online setting (Equation~\ref{equation: reinforce}) results in a quick collapse of performance due to the lack of any regularization to prevent out-of-distribution actions, as is well known in the offline RL problem setting outside of LLMs~\citep{kumar2019stabilizing}. Combining Equation~\ref{equation: reinforce} with an imitation learning loss stabilizes learning and results in a performance improvement, but still underperforms advantage-weighted regression (AWR)~\citep{awr}. Finally, we replaced the IQL in-sample expectile backup with a SARSA backup, where the utterance present in the offline dataset at the next turn is used to compute the Bellman target, i.e., no implicit or explicit maximization over target values is utilized, and the value function is trained to represent the long-term Q-values of the data collection policy. Observe that this variant did not offer the same level of policy improvement as using IQL to train the critic in this setting. This highlights the importance of maximization over actions to calculate Bellman targets in the offline setting. Finally, we also find that instantiations of ArCHer that use IQL and SARSA in conjunction with AWR, both outperform the na\"ive BC and filtered BC, further higlighting the importance of dynamic programming to train the critic.

\textbf{To summarize}, our experiments show that ArCHer can be used to derive multi-turn RL methods that lead to substantial improvements to sample efficiency of LLM policy training, benefit from offline and off-policy experience as well as improvements to RL algorithms, and scale with model capacity.

\vspace{-0.2cm}
\section{Discussion and Conclusion}
\vspace{-0.2cm}

In this paper, we propose a novel Actor-Critic Framework with a Hierarchical Structure (ArCHer) for multi-turn LLM agent tasks. By running two RL algorithms simultaneously, one at the high level (i.e., utterances in our practical method) and one at the low level (i.e., tokens), ArCHer reduces task horizons while enjoying the ability to retain a compact token-level action space at the low level. These characteristics yield a more practical, efficient, and effective method for training LLMs to be effective decision-makers or agents. ArCHer is simple and extensible, and can be flexibly instantiated with a variety of components for both the high- and low-level methods. Empirically, we observed that ArCHer significantly outperforms prior RL methods for LLMs, on a range of online RL tasks and scales favorably with more capable base models and other design improvements. 

Due to the computational constraints associated with running many trials of multiple RL algorithms and ArCHer instantiations, we had to conduct most of our experiments with a relatively small GPT-2 architecture. While our result with a Mistral7B base model demonstrates favorable scaling properties of our approach but rigorously evaluating our method with larger models (and on other benchmarks) is an important direction for future work. Our evaluations also focus entirely on tasks with computational rewards, and our method still requires a significant number of interactions (in the thousands), so an important future direction is to study how such methods can be made feasible to learn from live interactions with humans, when only about 100 interactions are available. We believe that model-based RL approaches could be quite promising here. Finally, deriving and evaluating novel practical algorithms from the hierarchical framework of ArCHer is also an interesting avenue for future work with the potential to greatly improve task performance.

\vspace{-0.2cm}
\section*{Acknowledgements}
\vspace{-0.2cm}
We thank Seohong Park, Mitsuhiko Nakamoto, Joey Hong, Yi Su, and members of the RAIL lab at UC Berkeley for valuable feedback on a previous version of this paper. YZ and AZ thank Marwa Abdulhai, Isadora White, and Charlie Snell for help with debugging and setting up the LMRL benchmark. AK thanks Charlie Snell, George Tucker, and Young Geng for informative discussions and perspectives. This work is supported by NSF IIS-2246811 and ONR N00014-21-1-2838, TRC TPU credit donations from Google Cloud, and compute credit donations from the Center for AI Safety (CAIS). AZ is supported by a postdoctoral fellowship from the Simons Institute for the Theory of Computing.

\vspace{-0.2cm}
\section*{Author Contributions}
\vspace{-0.2cm}
This work was done at UC Berkeley. {Yifei Zhou} led the project. He wrote the prototype of ArCHer, iterated on refining and improving it, implemented the baselines for online comparisons, and different ablation in the online experiments. He also took a lead in writing the paper. {Andrea Zanette} led the offline experiments of the project. He developed the offline variant of ArCHer with IQL and AWR, and helped editing the paper. {Jiayi Pan} set up the environment of webshop for our use and implemented prompting baselines. He also refactored the code and helped set up data distributed parallel training for ArCHer and baselines. {Sergey Levine} advised the project, provided inputs for prototyping the method, and helped editing the paper. {Aviral Kumar} proposed the project idea, advised the project, helped the prototyping of ArCHer, and substantially edited the paper.

\bibliography{arxiv_paper}

\newpage
\appendix
\onecolumn
\part*{Appendices}
\section{Environment and Dataset Details} \label{app: dataset_details}
In this section, we provide more details on the environments and datasets that we used in our experiments. Example actions and observations for each environment are shown in Figure \ref{fig: environment_example}.

\textbf{Detective Game~\citep{textscifi}.} In this game, the agent plays the role of a detective looking into a case where the Mayor got murdered. At each time step, the agent generates a free-form text where the game engine parses the text and determines the next state at each time step. The game engine provides a feedback of ``Illegal Action.'' if the generated text is an illegal action or cannot be correctly parsed. The optimal policy takes 51 steps to solve and reaches a maximum reward of 360. The game timeouts and terminates after 60 steps (including steps where illegal actions are generated). The observation at each time step includes the current surroundings, items carried, environment feedback for the outcome of the last action, and a list of available actions. The Supervised Fine-Tuning (SFT) dataset for this environment consists of 1000 trajectories of an agent picking a random action from the list of available actions at each timestep.

\begin{figure}[h!]  
    \centering
    \includegraphics[width=.75\linewidth]{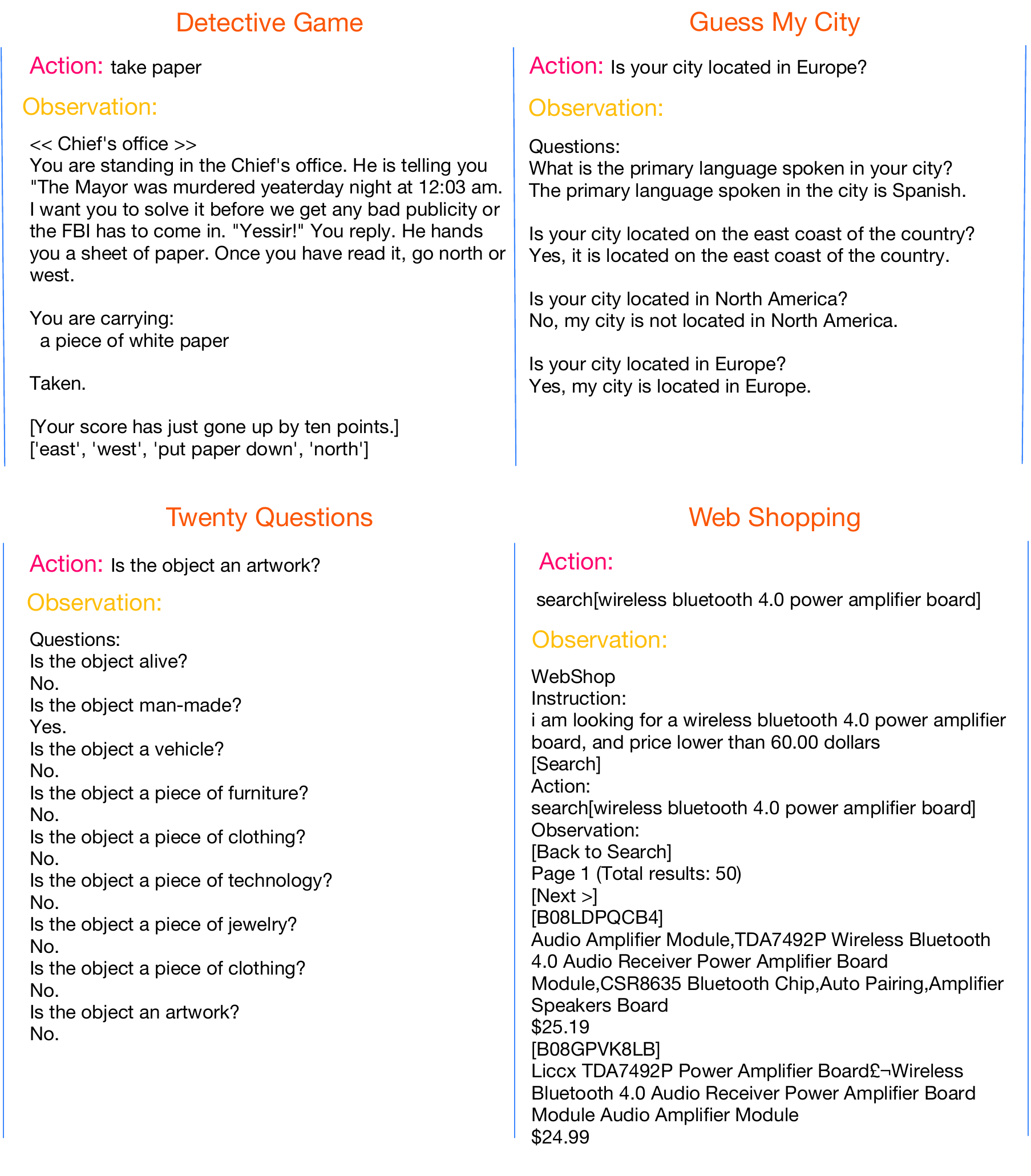}
    \caption{Example actions and observations for each environment that we study in this paper.}
    \label{fig: environment_example}
\end{figure}

\textbf{Twenty Questions and Twenty Questions Subset~\citep{lmrl}.} In this environment, for each episode, a random word is chosen from a list of 157 words of household items such as ``basketball'', ``apple'', and ``car''. The word is held hidden from the agent and the agent is tasked to guess the hidden word within 20 questions. The questions are limited to yes/no questions and the answers from the oracle are limited to ``Yes.'', ``No.'', and ``Invalid Question.''. As opposed to using ``flan-t5-xl''\citep{chung2022scaling} as the oracle \citep{lmrl}, we train a `flan-t5-small'' to simulate the oracle with the same data and use it for our online experiments due to computational constraints. The agent gets a reward of 0 if it guesses the correct word and the episode terminates. Otherwise, the agent gets a reward of -1 for each question it raises. This reward structure results in a minimum reward of -20 if the agent does not guess the corret word with twenty questions and a maximum reward of 0 if the agent guesses the correct word with the first question although it is very unlikely. We use the official offline dataset provided by \citet{lmrl} with 100K simulated episodes. Our SFT checkpoints for online experiments for both Twenty Questions and Twenty Questions Subset are also trained with this dataset. Twenty Questions Subset keeps everything else the same except that it uses a subset of 10 hidden words in the word list. Since the offline dataset and the SFT checkpoint for online experiments are based on the entire Twenty Questions, Twenty Questions Subset challenges different algorithms with a significant distribution shift and requires the agent to come up with an entirely different strategy from behavior cloning. 

\textbf{Guess My City~\citep{lmrl}.} This environment is a similar dialogue task to Twenty Questions. For each episode, a random city is chosen from a list of 100 cities in the world. The city is held hidden from the agent and the agent is tasked to guess the name of the city within 20 questions. Both the questions and answers are free-form except that the answers are not allowed to contain the name of the city. As opposed to using ``flan-t5-xl''\citep{chung2022scaling} as the oracle \citep{lmrl}, we train a `flan-t5-small'' to simulate the oracle with the same data and use it for our online experiments due to computational constraints. We found in our online experiments that the agent can easily learn to ``exploit'' the oracle by tricking it to directly output the name of the city. Therefore, we simply replace the answer with a hardcoded template ``I cannot answer that question.'' if the name of the city is found in the output of the oracle language model to reduce reward hacking. The reward structure is the same as Twenty Questions. We use the official offline dataset provided by \citet{lmrl} with 100K simulated episodes.

\textbf{Web Shopping~\citep{webshop}.} This environment challenges the ability of the agents to interact with external tools. For each episode, a random instruction requesting a specific item is chosen and shown to the agent. The agent needs to make use of a simplified web shopping server to make the purchase. Every successful purchase is consisted of searching the keywords in the search engine, selecting an item from searched results, clicking on features and attributes for the item, and finally making the purchase. Following ReAct~\citep{yao2023react}, the agent can choose to take a ``think'' action before taking any actual actions such as ``search'' and ``click''. An observation consists of the instruction and the history of visited webpages (described in text) and actions. The reward is a scalar between 0 and 1 depending on the similarity of the purchased item with the requested item. For example, a partial reward will be given if the agent purchases a black king-sized bed while a black queen-sized bed is requested. The episode timeouts after 10 interaction steps and a reward of 0 is issued. Our main online environments use a subset of 100 instructions from index 2000 to 2100 for a fast evaluation. We collect the offline dataset using the instructions from index 0 to 1000 with GPT-3 text-davinci-002 with prompts from ReAct's official implementation.
\section{Offline Algorithm and Practical Considerations} \label{section: offline_details}
Our offline algorithm is a hierarchical version of
the IQL algorithm \citep{iql}. 
Specifically, the critic leverages IQL (\cref{equation: iql}) while the actor update is based on AWR (Equation \ref{equation: awr}). 

These choices for the actor and for the critic update identify two key hyperparameters, the expectile value \expectile{} (defined in Equation \ref{equation: iql} and \ref{equation: awr}) and the temperature \invtemp{}, whose effect is described in the respective sections.
These hyper-parameters are already present in the original IQL algorithm \citep{iql}, and they have a similar interepretation here.
By choosing \expectile{} and \invtemp{} appropriately, the algorithm identifies a policy whose performance should be between the optimal one and the one that generated the dataset. (In general, recovering the optimal policy by just using a dataset may not be possible as the dataset may not contain information about an optimal policy).  

The offline algorithm shares most of the ingredients with its online counterpart, such as the double critic, target networks, soft updates, and value function heads.
However, some unique features inherited from IQL allow to considerably simplify several algorithmic choices.
\begin{itemize}
    \item The actor and the critic no longer need to be synchronized by using a certain update ratio.
    This is because the critic update defined in Equation \ref{equation: iql} is independent of the actor's current policy, and so the two can be updated with any desired frequency without introducing instabilities.
    \item It is not necessary to pre-train the policy with a behavioural cloning objective, because such objective is already included in the actor's loss function in Equation \ref{equation: awr}. 
    \item The warmup steps for the critic are also not necessary, because the initially small advantage function has a neglegible effect in the AWR loss.
\end{itemize}
\section{Additional Baseline Details}
\subsection{Performance of PPO}
In Figure \ref{fig: ppo}, we provide a zoom-in of the learning curves of PPO for Twenty Questions, Twenty Questions Subset, and Guess My City. We observed that PPO does improve over the SFT checkpoint, especially in the more simple task Twenty Questions Subset. However, as PPO is unable to reuse past off-policy data, we need to collect at least 1024 trajectories of on-policy data for each PPO update, as shown in Appendix \ref{app: hyperparameters}. This observation is consistent with \citet{lmrl}. 
\begin{figure*}[h!]  
    \centering    \includegraphics[width=.75\linewidth]{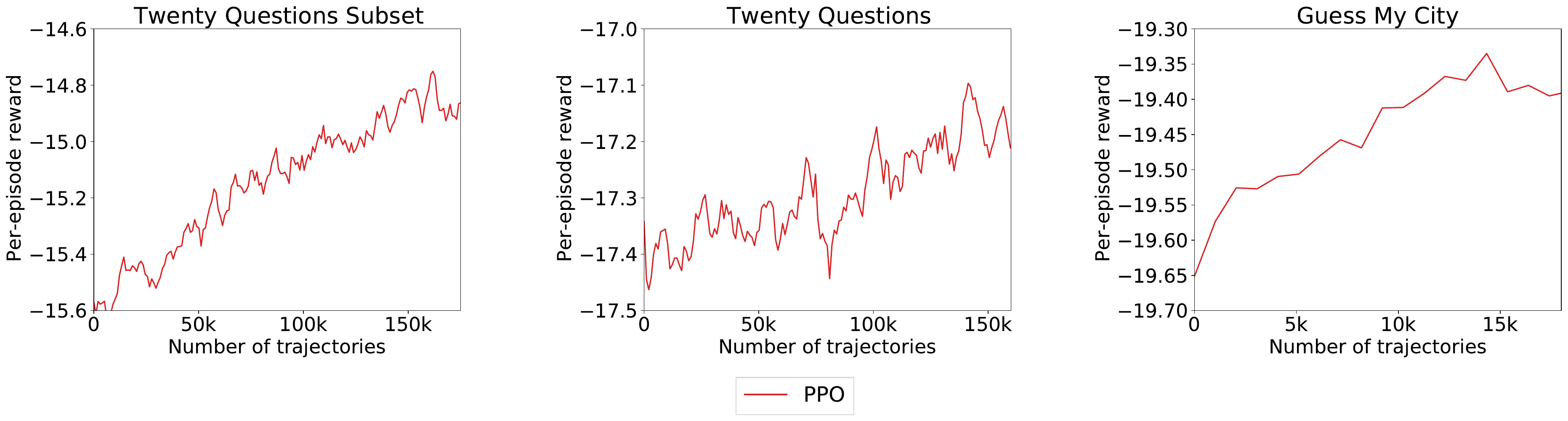}
    \caption{A zoom-in of learning curves of PPO. PPO gradually improves despite its worse sample complexity compared to other off-policy methods.}
    \label{fig: ppo}
\end{figure*}

\subsection{Additional Reproduction Details for WebShop Experiment}
For our WebShop experiment, we utilized the environment and the few-shot prompting baselines from ReAct \citep{yao2023react}. ReAct introduces two prompting strategies: the ReAct method and the Act-only method, which we denote as \emph{ReAct} and \emph{expert-prompt} in Figure \ref{fig: webshop}. The ReAct method additionally allows the agent to articulate its reasoning before making an action, whereas the Act-only method does not. We use the original prompts in ReAct's implementation without any modifications, and to ensure that our paper is self-contained, we also include these prompts in Figure \ref{fig:act-prompt} and \ref{fig:react-prompt}.

We assessed these algorithms on the webshop index from 2000 to 2100 to establish the LLM-based few-shot baselines in Figure \ref{fig: webshop}. Notably, by early January 2024, OpenAI had deprecated the GPT-3 text-davinci-002 model used in the original ReAct study. Consequently, we switched to the gpt-3.5-turbo-instruct model. While the expert-prompt approach yielded performance comparable to the best performing text-davinci-002 based method, the ReAct method underperformed significantly with the newer model. This discrepancy, also observed and replicated in other tasks independetly by other researchers \footnote{For more details, see the Github Issue on ReAct's repository: \url{https://github.com/ysymyth/ReAct}}, is faithfully depicted in Figure \ref{fig: webshop}.
\begin{figure}[H]
    \centering
    \setlength{\fboxrule}{0.8pt}
    \fbox{
        \parbox{0.9\textwidth}{
            \textbf{ReAct Prompt}\\
Webshop 

Instruction:  

i would like a 3 ounce bottle of bright citrus deodorant for sensitive skin, and price lower than 50.00 dollars 

[Search]

Action: search[3 ounce bright citrus deodorant sensitive skin]

Observation: 

[Back to Search] 

Page 1 (Total results: 50) 

[Next $>$] 

[B078GWRC1J] 

Bright Citrus Deodorant by Earth Mama | Natural and Safe for Sensitive Skin, Pregnancy and Breastfeeding, Contains Organic Calendula 3-Ounce 

\$10.99 

[B078GTKVXY] 

Ginger Fresh Deodorant by Earth Mama | Natural and Safe for Sensitive Skin, Pregnancy and Breastfeeding, Contains Organic Calendula 3-Ounce 

\$10.99 

[B08KBVJ4XN] 

Barrel and Oak - Aluminum-Free Deodorant, Deodorant for Men, Essential Oil-Based Scent, 24-Hour Odor Protection, Cedar \& Patchouli Blend, Gentle on Sensitive Skin (Mountain Sage, 2.7 oz, 2-Pack) 

\$15.95

Action: think[B078GWRC1J and B078GTKVXY are bright citrus deodorant less then 50 dollars. I can check B078GWRC1J first.]

Observation: OK.

Action: click[B078GWRC1J]

Observation: 

[Back to Search] 

[$<$ Prev] 

scent [assorted scents][bright citrus][calming lavender][ginger fresh][simply non-scents]

size [travel set (4-pack)][3 ounce (pack of 1)][3-ounce (2-pack)]

Bright Citrus Deodorant by Earth Mama | Natural and Safe for Sensitive Skin, Pregnancy and Breastfeeding, Contains Organic Calendula 3-Ounce 

Price: \$10.99 

Rating: N.A. 

[Description] 

[Features] 

[Reviews] 

[Buy Now]

Action: think[For 3 ounce bottle of bright citrus deodorant for sensitive skin, the item has options 'bright citrus' and '3 ounce (pack of 1)' and seems good to buy.]

Observation: OK.

Action: click[bright citrus]

Observation: You have clicked bright citrus.

Action: click[3 ounce (pack of 1)]

Observation: You have clicked 3 ounce (pack of 1).

Action: click[Buy Now]
        }
    }
    \caption{Prompts used by the act-only (expert) baseline.}
    \label{fig:react-prompt}
\end{figure}
\begin{figure}[H]
    \centering
    \setlength{\fboxrule}{0.8pt}
    \fbox{
        \parbox{0.9\textwidth}{
            \textbf{Act-only (expert) Prompt}\\
Instruction:  

i would like a 3 ounce bottle of bright citrus deodorant for sensitive skin, and price lower than 50.00 dollars 

[Search]

Action: search[3 ounce bright citrus deodorant sensitive skin]

Observation: 

[Back to Search] 

Page 1 (Total results: 50) 

[Next $>$] 

[B078GWRC1J] 

Bright Citrus Deodorant by Earth Mama | Natural and Safe for Sensitive Skin, Pregnancy and Breastfeeding, Contains Organic Calendula 3-Ounce 

\$10.99 

[B078GTKVXY] 

Ginger Fresh Deodorant by Earth Mama | Natural and Safe for Sensitive Skin, Pregnancy and Breastfeeding, Contains Organic Calendula 3-Ounce 

\$10.99 

[B08KBVJ4XN] 

Barrel and Oak - Aluminum-Free Deodorant, Deodorant for Men, Essential Oil-Based Scent, 24-Hour Odor Protection, Cedar \& Patchouli Blend, Gentle on Sensitive Skin (Mountain Sage, 2.7 oz, 2-Pack) 

\$15.95

Action: click[B078GWRC1J]

Observation: 

[Back to Search] 

[$<$ Prev] 

scent [assorted scents][bright citrus][calming lavender][ginger fresh][simply non-scents]

size [travel set (4-pack)][3 ounce (pack of 1)][3-ounce (2-pack)]

Bright Citrus Deodorant by Earth Mama | Natural and Safe for Sensitive Skin, Pregnancy and Breastfeeding, Contains Organic Calendula 3-Ounce 

Price: \$10.99 

Rating: N.A. 

[Description] 

[Features] 

[Reviews] 

[Buy Now]

Action: click[bright citrus]

Observation: You have clicked bright citrus.

Action: click[3 ounce (pack of 1)]

Observation: You have clicked 3 ounce (pack of 1). 

Action: click[Buy Now]
        }
    }
    \caption{Prompts used by the act-only (expert) baseline.}
    \label{fig:act-prompt}
\end{figure}
\section{Additional Experimental Results}
\subsection{TD-Learning v.s. MC Regression}
To validate whether TD-learning plays an important role in ArCHer, we carried out an ablation study where we replaced TD-learning in ArCHer with MC regression for critic updates. To make sure that data in the replay buffer are generated by similar policies, we use a smaller replay buffer that contains trajectories collected by three latest policies. The ablation results are shown in Figure \ref{fig: mc_returns}. We observed that MC regression may learn faster in the beginning as the information propagates faster than per step TD learning, but it fails to learn reliably over then entire training process. This ablation result shows the importance of TD-learning to effectively make use of off-policy data.

\subsection{Online IQL Critic Loss}
In our ablation study of using IQL critic loss in the online setting, we set $\tau = 0.9$ to encourage more risk-seeking for better explorations. As we can see from Figure \ref{fig: iql_critic}, the use of IQL critic can indeed accelerate explorations in simple tasks such as Detective Game. However, a naive instantiation of IQL fails to provide an unbiased estimate for the policy gradient resulting in potential instabilities in harder tasks such as Twenty Questions.

\begin{figure}[ht!]
  \centering
  \begin{subfigure}[b]{0.4\textwidth}
    \centering
    \includegraphics[width=\linewidth]{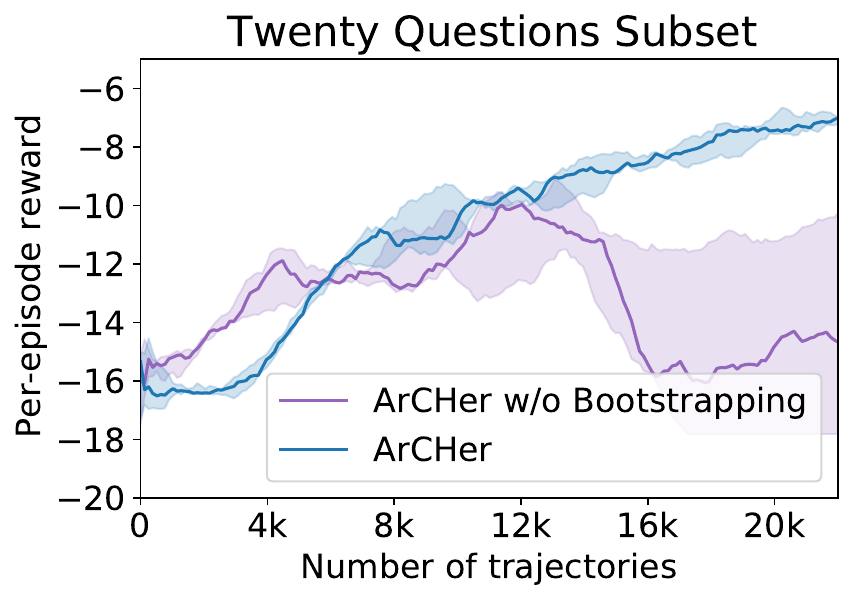}
    \caption{Ablation study of replacing ArCHer critic update with MC Regression. As data in the replay buffer is generated by different policies, MC regression fails to improve reliably.}
    \label{fig: mc_returns}
  \end{subfigure}
  \hfill
  \begin{subfigure}[b]{0.51\textwidth}
    \centering
    \includegraphics[width=\textwidth]{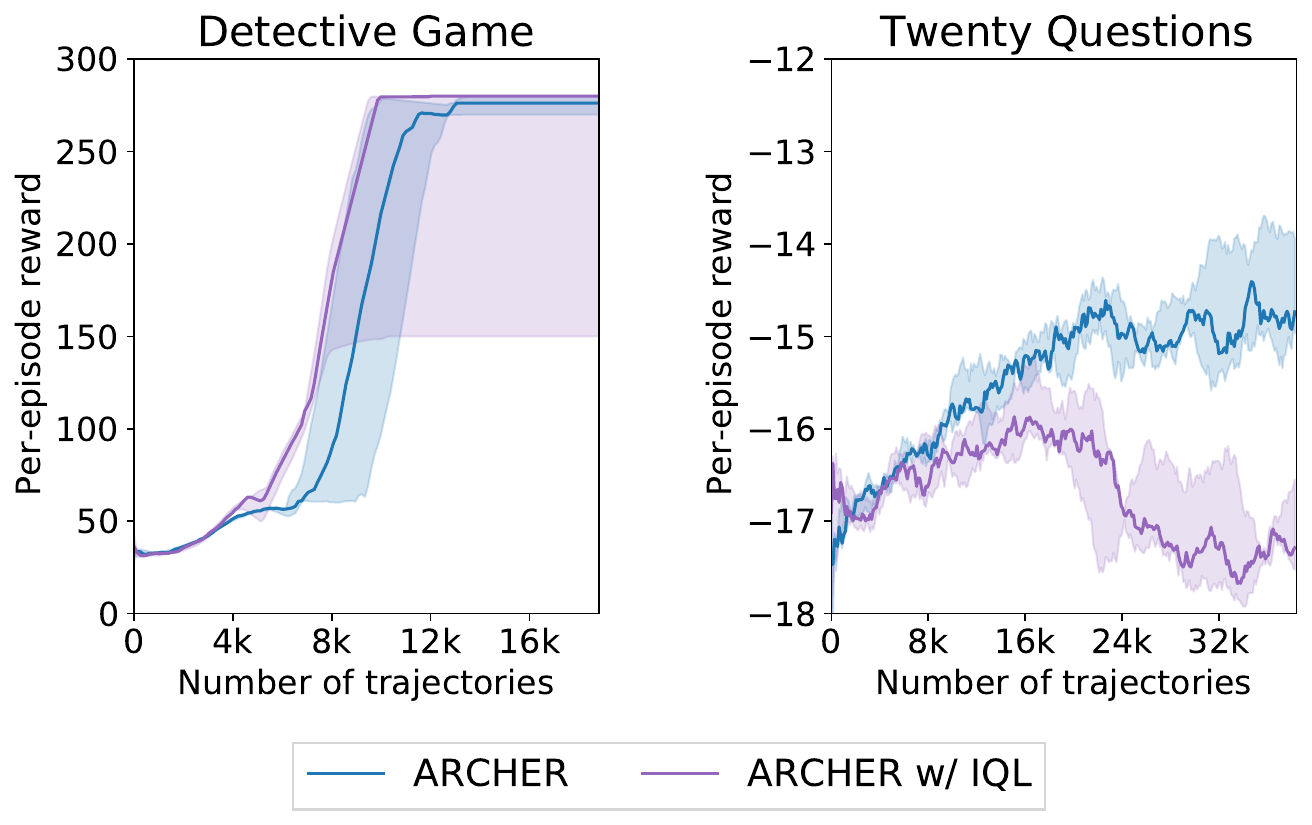}
    \caption{Ablation study of replacing ArCHer critic update with an IQL loss. Using an IQL critic loss can sometimes accelerate training but introduce more instability.}
    \label{fig: iql_critic}
  \end{subfigure}
\end{figure}
\section{Reward Hacking}\label{app: hacking}
To understand whether our agent has really learnt to behave more strategically with reinforcement learning instead of exploiting the environment, we present sample trajectories for Twenty Questions Subset, Twenty Questions, Guess My City, and WebShop in Figure \ref{fig:twenty_questions_subset}, \ref{fig:twenty_questions}, \ref{fig:guess_my_city}, \ref{fig:webshop_example}. Note that for detective game, there is only a limited number of legal actions and all illegal actions will lead to making no progress in the game, resulting in a lower rewards. For Twenty Questions Subset and Twenty Questions, we observed that the successful dialogues tend to be mostly natural and strategic while some unsuccessful dialogues tend to contain more meaningless repetitions. For Guess My City, we observed that the agent learnt to ask the predominant language of the city which is one of the most informative questions in the SFT distribution we considered in Appendix \ref{app: dataset_details}. However, as there are some vulnerabilities and hallucinations of the imperfect oracle model for generating free-form answers, some of the conversations might diverge to meaningless repetitions in the end. For WebShop, as the environment forces the format of legal actions for the web server to process, the trajectories tend to be natural and realistic. 
\begin{figure}[H]
    \centering
    \setlength{\fboxrule}{0.8pt}
    \fbox{
        \parbox{0.9\textwidth}{
            \textbf{Twenty Questions Subset}\\
            Questions:\\
            Is it a living thing? No.\\
            Is it a man-made? Yes.\\
            Is it a machine? No.\\
            Is it a tool? No.\\
            Is it a piece of clothing? No.\\
            Is it a piece of furniture? Yes.\\
            Is it a chair? Yes.
        }
    }
    \caption{Example environment interactions of ArCHer for Twenty Questions Subset at the end of training. It can be seen that in this environment, ArCHer learnt a near-optimal binary search strategy.}
    \label{fig:twenty_questions_subset}
\end{figure}
\begin{figure}[H]
    \centering
    \setlength{\fboxrule}{0.8pt}
    \fbox{
        \parbox{0.9\textwidth}{
            \textbf{Twenty Questions}\\
            Questions:\\
            Is it a living thing? Yes.\\
            Is the object a mammal? No.\\
            Is the object a plant? Yes.\\
            Is the object edible? Yes.\\
            Is the object a vegetable? No.\\
            Is the object a fruit? Yes.\\
            Is the object a citrus fruit? No.\\
            Is the object a berry? Yes.\\
            Is the object a strawberry? No.\\
            Is the object a strawberry? No.\\
            Is the object a blueberry? No.\\
            Is the object a rasepeberry? No.\\
            Is the object a currant? No.\\
            Is the object a grape? Yes.\\
            \\
            \\
            Questions:\\
            Is the object a living thing? No.\\
            Is the object man-made? Yes.\\
            Is the object a vehicle? No.\\
            Is the object a piece of jewelry? No.\\
            Is the object a musical instrument? Yes.\\
            Is the object a stringed instrument? No.\\
            Is the object a wind instrument? Yes.\\
            Is the object a trumpet? No.\\
            Is the object a trombone? No.\\
            Is the object a flute? No.\\
            Is the object a saxophone? No.\\
            Is the object a clarinet? No.\\
            Is the object a clarinet? No.\\
            Is the object a clarinet? No.\\
            Is the object a clarinet? No.\\
            Is the object a clarinet? No.\\
            Is the object a clarinet? No.\\
            Is the object a clarinet? No.\\
            Is the object a clarinet? No.\\
            Is the object a clarinet? No.\\

        }
    }
    \caption{Example environment interactions of ArCHer for Twenty Questions (one successful and one unsuccessful trajectory) at the end of training. In this environment, the conversation is smooth and natural for successful trajectories. However, for unsuccessful trajectories, the agent may sometimes keep repeating the same question.}
    \label{fig:twenty_questions}
\end{figure}

\begin{figure}[H]
    \centering
    \setlength{\fboxrule}{0.8pt}
    \fbox{
        \parbox{0.9\textwidth}{
            \textbf{Guess My City}\\
            Questions:\\
            What is the predominant language spoken: \\
            The predominant language spoken is Mandarin Chinese.\\
            Is the city you are from Shanghai? \\
            No, I am not from Shanghai.\\
            Is the city you are from Taipei,Is the city you are from TaipeiIs the city you are from Tis aipeiIs the city you are from Tai\\ Yes\\
            \\
            \\
            Questions:\\
            What is the predominant language spoken: \\
            English\\
            Is the city you are from London?\\
            No, I am not from London.\\
            Is the city you are from LondonIs the city you are from London Is the city you are from London\\
            Yes, my city is from London.\\
            Is the city you are from LondonIs the city you are from London Is the city you are from London\\
            Yes, my city is from London.\\
            Is the city you are from LondonIs the city you are from London Is the city you are from London\\
            Yes, my city is from London.\\
            Is the city you are from LondonIs the city you are from London Is the city you are from London\\
            Yes, my city is from London.\\
            Is the city you are from LondonIs the city you are from London Is the city you are from London\\
            Yes, my city is from London.\\
            ......
        }
    }
    \caption{Example environment interactions of ArCHer for Guess My City at the end of training. The agent learns to ask the question for the predominant language, which is one of the most informative questions from the offline SFT distribution. However, as the oracle model provides free-form responses can be out of distribution, the agent is more prone to collapsing in out-of-distribution states.}
    \label{fig:guess_my_city}
\end{figure}
\begin{figure}[H]
    \centering
    \setlength{\fboxrule}{0.8pt}
    \fbox{
        \parbox{0.9\textwidth}{
            \textbf{WebShop}\\
            Observation:\\
            WebShop\\
            Instruction:\\
            i would like a citrus yao conditioner made with natural ingredients, and price lower than 50.00 dollars\\
            $[$Search$]$\\
            Action:\\
            search$[$citrus yao conditioner natural ingredients$]$\\
            Observation:\\
            $[$Back to Search$]$\\
            Page 1 (Total results: 50)\\
            $[$Next $>$$]$\\
            $[$B099KWBV49$]$\\
            VIORI Terrace Garden Conditioner Bar - Handcrafted with Longsheng Rice Water \& Natural Ingredients - Sulfate-free, Paraben-free, Cruelty-free, Phthalate-free, pH balanced 100\% Vegan, Zero-Waste\\
            \$16.25\\
            $[$B09KQZ9GTK$]$\\
            VIORI Hidden Waterfall Shampoo \& Conditioner Bar Set WITHOUT Bamboo Holder - Handcrafted with Longsheng Rice Water \& Natural Ingredients - Sulfate-free, Paraben-free, Phthalate-free, 100\% Vegan\\
            \$29.25\\
            Action:\\
            click$[$B099KWBV49]\\
            Observation:\\
            $[$Back to Search$]$\\
            $[$$<$ Prev$]$\\
            style $[$bamboo bar holder$]$$[$citrus yao$]$$[$hidden waterfall$]$\\
            size $[$bamboo bar holder$]$$[$conditioner$]$$[$shampoo$]$$[$shampoo + conditioner set$]$\\
            VIORI Terrace Garden Conditioner Bar - Handcrafted with Longsheng Rice Water \& Natural Ingredients - Sulfate-free, Paraben-free, Cruelty-free, Phthalate-free, pH balanced 100\% Vegan, Zero-Waste\\
            Price: \$16.25\\
            Rating: N.A.\\
            $[$Description$]$\\
            $[$Features$]$\\
            $[$Reviews$]$\\
            $[$Attributes$]$\\
            $[$Buy Now$]$\\
            Action:\\
            click$[$citrus yao$]$\\
            Observation:\\
            You have clicked citrus yao.\\
            Action:\\
            click$[$Buy Now$]$\\
            Observation:\\
            Your score (min 0.0, max 1.0): 0.75\\
        }
    }
    \caption{Example environment interactions of ArCHer for WebShop at the end of training. It can be seen that in this environment, the behavior of ArCHer is very natural.}
    \label{fig:webshop_example}
\end{figure}
\section{Hyperparameters} \label{app: hyperparameters}
Table \ref{table: hyperparameters_combination_lock} shows hyperparameters for ArCHer and other baselines for all environments. As shown in the table, most hyperparameters are held the same across all environment except that a smaller actor learning rate is used for Twenty Questions and a larger number of rollout trajectories is used for Web Shopping. This shows that ArCHer is relatively insensitive to selection of hyperparameters and does not require too much hyperparameter tuning to work in different environments.
\begin{table}[h] 
\caption{Hyperparameters for All Experiments} 
\centering
\resizebox{\linewidth}{!}{  
\begin{tabular}{c|c|ccccc} 
\toprule
& & Detective Game& Twenty Questions Subset& Twenty Questions& Guess My City& Web Shopping\\
\hline
\multirow{2}{4em}{SFT} & actor lr & 2e-4 & 2e-4 & 2e-4 & 2e-4 & 2e-4\\
& batch size& 32 & 32 & 32 & 32 & 32\\
\hline
\multirow{6}{4em}{Filtered BC} & actor lr & 3e-4 & 3e-4 & 3e-5 & 3e-4 & 3e-4\\
& batch size & 256 & 256 & 256 & 256 & 256\\
& rollout trajectories & 32 & 32 & 32 & 32 & 128\\
& replay buffer size & 10000 & 10000 & 10000 & 10000 & 10000\\
& filtering percentage & 0.1 & 0.1 & 0.1 & 0.1 & 0.1\\
& actor updates per iteration & 10 & 10 & 10 & 10 & 10\\
\hline
\multirow{8}{4em}{CHAI} & actor lr & 3e-4 & 3e-4 & 3e-5 & 3e-4 & 3e-4\\
& critic lr & 6e-4 & 6e-4 & 6e-4 & 6e-4 & 6e-4\\
& batch size & 256 & 256 & 256 & 256 & 256\\
& rollout trajectories & 128 & 128 & 128 & 128 & 512\\
& replay buffer size & 10000 & 10000 & 10000 & 10000 & 10000\\
& critic updates per iteration & 50 & 50 & 50 & 50 & 50\\
& discount & 0.98 & 0.95 & 0.95 & 0.95 & 0.9\\
& polyak alpha & 0.9 & 0.9 & 0.9 & 0.9 & 0.9\\
\hline
\multirow{10}{4em}{ArCHer} & actor lr & 3e-4 & 3e-4 & 3e-5 & 3e-4 & 3e-4\\
& critic lr & 6e-4 & 6e-4 & 6e-4 & 6e-4 & 6e-4\\
& batch size & 256 & 256 & 256 & 256 & 256\\
& rollout trajectories & 128 & 128 & 128 & 128 & 512\\
& replay buffer size & 10000 & 10000 & 10000 & 10000 & 10000\\
& critic updates per iteration & 50 & 50 & 50 & 50 & 50\\
& discount & 0.98 & 0.95 & 0.95 & 0.95 & 0.9\\
& polyak alpha & 0.9 & 0.9 & 0.9 & 0.9 & 0.9\\
& actor updates per iteration & 3 & 3 & 3 & 3 & 3\\
& warm up iters with no actor update & 10 & 10 & 20 & 10 & 20\\
\hline
\multirow{10}{4em}{ArCHer w/ Baseline} & actor lr & \textbackslash & \textbackslash & \textbackslash & 3e-4 & \textbackslash \\
& critic lr & \textbackslash & \textbackslash & \textbackslash & 6e-4 & \textbackslash\\
& batch size & \textbackslash & \textbackslash & \textbackslash & 256 & \textbackslash\\
& rollout trajectories & \textbackslash & \textbackslash & \textbackslash & 128 & \textbackslash \\
& replay buffer size & \textbackslash & \textbackslash & \textbackslash & 10000 & \textbackslash\\
& critic updates per iteration & \textbackslash & \textbackslash & \textbackslash & 50 & \textbackslash\\
& discount & \textbackslash & \textbackslash & \textbackslash & 0.95 & \textbackslash\\
& actor updates per iteration & \textbackslash & \textbackslash & \textbackslash & 3 & \textbackslash\\
& baseline updates per iteration & \textbackslash & \textbackslash & \textbackslash & 60 & \textbackslash\\
& warm up iters with no actor update & \textbackslash & \textbackslash & \textbackslash & 10 & \textbackslash\\
& polyak alpha & \textbackslash & \textbackslash & \textbackslash & 0.9 & \textbackslash\\
\hline
\multirow{9}{4em}{PPO} & actor lr & \textbackslash & 6e-6 & 6e-6 & 6e-4 & \textbackslash \\
& batch size & \textbackslash & 1024 & 1024 & 1024 & \textbackslash\\
& rollout trajectories & \textbackslash & 2048 & 2048 & 1024 & \textbackslash \\
& PPO epochs & \textbackslash & 10 & 20 & 4 & \textbackslash\\
& discount & \textbackslash & 0.95 & 0.95 & 0.95 & \textbackslash\\
& GAE lambda & \textbackslash & 0.95 & 0.95 & 0.95 & \textbackslash\\
& clip range & \textbackslash & 0.2 & 0.2 & 0.2 & \textbackslash\\
\toprule
\end{tabular}}\label{table: hyperparameters_combination_lock}
\caption{Hyperparameters for ArCHer and baseline methods for all experiments.}
\end{table}

\section{Proof of Main Theorem}
\label{proofs}
\subsection{Equivalent Utterance and Token Level MDPs}
We consider the groundtruth utterance-level discounted infinite horizon MDP $\Mcal =\{\Scal,\Acal, \gamma, r, \mu_0, P\}$ as defined in Section \ref{section: hierarchical_mdp}. An equivalent token-level infinite horizon MDP can be constructed with $\widetilde{\Mcal} = \{\widetilde{\Scal}, \widetilde{\Acal}, \gamma^{1/L}, r, \mu_0, \widetilde{P}\}$, where $L$ is the length of each utterance, $\widetilde{\Acal}$ contains all individual tokens with a special padding token (such that each utterance is padded to the same length $L$), $\widetilde{\Scal} \in \Scal \times \widetilde{\Acal}^L$ contains both the state in the utterance-level MDP and the partial sequence of utterance that has already been generated, $\gamma^{1/L}$ is the equivalent discount factor. Note that this definition of token-level MDP is not the same as the definition of token-level MDP in the hierarchical language MDP defined for ArCHer in Section \ref{section: hierarchical_mdp}. The token-level MDP for ArCHer is only embedded in one particular utterance and the only reward that it receives is at the end of the utterance from the utterance-level critic while the equivalent token-level MDP spans multiple utterances and receives rewards from the external environment at the end of each utterance. This construction of token-level MDP is ``equivalent'' to the utterance-level MDP in the sense that for any autoregressive policy $\pi$ that generates an utterance token by token, we have:
\begin{align*}
    \forall s \in \Scal, V^\pi(s) = \widetilde{V}^\pi(s),
\end{align*}
where $V^\pi$ and $\widetilde{V}^\pi$ are value functions of $\pi$ in the utterance-level MDP $\Mcal$ and the token-level MDP $\widetilde{\Mcal}$ respectively. We use $\widetilde{\pi}$ for the same utterance-level policy $\pi$ when it generates one token at a time. 
 
As usual, for any MDP $\Mcal$, we define $d^\pi\in\Delta(\Scal\times\Acal)$ as the average state-action occupancy measure of policy $\pi$ such that:
\begin{align*}
    d^{\pi}(s, a) = (1 - \gamma) (\mu_0(s)\pi(a|s) + \sum_{t=1}^\infty \gamma ^t P^\pi(s_t = s, a_t = a))
\end{align*}We denote $V^\pi = \EE_{s_0\sim \mu_0} V^\pi(s_0)$ as the expected total discounted reward of $\pi$. We denote $\Tcal^\pi$ as the Bellman operator associated with $\pi$, i.e., given a function $f\in \Scal\times\Acal\mapsto \mathbb{R}$, we have 
\begin{align*}
\Tcal^\pi f(s,a) = r(s,a) + \gamma \EE_{s'\sim P(s,a),a'\sim \pi(s')} \brk*{f(s',a')}. 
\end{align*}
Similar definitions can be made in the token-level MDP $\widetilde{\Mcal}$ for $\widetilde{d}^\pi$, $\widetilde{V}^\pi$, and $\widetilde{\Tcal}^\pi$. We also define $\widetilde{A}^\pi$ as the advantage function in the token level.

\subsection{Fitted Policy Evaluation Subroutine}
In this section, we present our theoretical subroutine for fitting the critic in Algorithm \ref{alg:FPE}. On a high level, it just repeats finding the Q function that minimizes the bellman error with respect to the Q function in the last iteration, and returns the average of all Q functions in the end. Both critic fitting in the token-level MDP $\widetilde{\Mcal}$  or in the utterance-level MDP $\Mcal$ follows from the same subroutine with the same function class $\Fcal$ that map the space of seuqnces of tokens to real values. This theoretical algorithm is simply a more fomalized version of the critic update in Algorithm \ref{alg:archer_detail}.
\begin{algorithm}[h]
\begin{algorithmic}[1] 
\Require Policy $\pi$, function class $\Fcal$, number of iterations $K$, weight $\lambda$ 
\Require $K$ independent datasets \(\Dcal_{1:K} = \{(s,a, r, s')\}\) of $M$ many samples each from the same offline distribution $\nu$\label{line:don_sample}. 
\State Initialize $f_0 \in \Fcal$. 
\For{$k = 1, \dots, K$} 
\State Solve the square loss regression problem to compute: 
\begin{align*} 
f_k \leftarrow \argmin_{f\in\Fcal} ~ \widehat \EE_{\Dcal_k}[( f(s,a) - r - \gamma f_{k-1}(s', \pi(s')))^2] \numberthis \label{eq:fitted_Q_eqn} 
\end{align*}  
\EndFor  
\State \textbf{Return} $\bar f = \tfrac{1}{K} \sum_{k=1}^{K} f_{k}$.
\end{algorithmic}
\caption{\textbf{F}itted \textbf{P}olicy \textbf{E}valuation (\textbf{FPE})} 
\label{alg:FPE} 
\end{algorithm}

\subsection{Assumptions}
We present the two important assumptions that we use for analyzing FPE subroutine, and both assumptions share the same definition for utterance-level MDP $\Mcal$ and token-level MDP $\widetilde{\Mcal}$.
\begin{assumption}[Bellman Completeness~\citep{song2023hybrid, zhou2023offline, zanette2023realizability, xie2021bellman}]
\label{def:inherent_BE_error} We say that \(
\Fcal
\) is Bellman Complete for some policy $\pi$, if for all \(f \in \Fcal\), there exists a \(f' \in \Fcal\) such that \(\|f'(s,a) - \Tcal^\pi f(s,a)\|_{\infty} = 0\). 
\end{assumption} 
\begin{assumption}[Density Ratio Coefficient~\citep{zhan2022offline, foster2021offline}] 
\label{def:offline_coverage} 
Given the offline distribution \(\nu\), for any policy $\pi$, we define the density ratio coefficient as 
 \begin{align*} 
C_{\nu,\pi} :=  \max_{s,a}  \frac{d^\pi(s,a)}{ \nu(s,a)}. %
\end{align*} 
\end{assumption}

The following two lemmas compare the utterance-level assumptions and token-level assumptions.
\begin{lemma}\label{lem: bellman_complete}
    For any stationary policy $\pi$, token-level Bellman Completeness for $\widetilde{\Mcal}$ $\implies$ utterance-level Bellman Completeness for $\Mcal$
\end{lemma}
\begin{proof}
    $\forall s \in \Scal, a \in \Acal$ being state and action in the utterance-level, the utterance $a$ can be decomposed into $L$ tokens in the action space of the token-level MDP:
    \begin{align*}
        a = \widetilde{a}_{1:L}, \widetilde{a}_i \in \widetilde{\Acal}, i = 1,\dots,L.
    \end{align*}
    Therefore, 
    \begin{align*}
        &\min_{f' \in \Fcal} |f'(s,a) - \Tcal^\pi f'(s,a)| \\
        = &  \min_{f_1,\dots, f_L \in \Fcal}|f_1(s,a) - \widetilde{\Tcal}^\pi f_2(s,a) + r(s,a) + \gamma^{1/L} \EE_{s' \sim P(\cdot|s,a), \widetilde{a}_1 \sim \widetilde{\pi}(\cdot|s')}f_2(s', \widetilde{a}_1) \\
        & -\gamma^{1/L} \EE_{s' \sim P(\cdot|s,a), \widetilde{a}_1 \sim \widetilde{\pi}(\cdot|s')}\widetilde{\Tcal}^\pi f_3(s', \widetilde{a}_1) + \dots + \gamma^{(L-1)/L} \EE_{s' \sim P(\cdot|s,a), \widetilde{a}_{1:L-1} \sim \widetilde{\pi}(\cdot|s')}f_L(s', \widetilde{a}_{1:L-1})\\
        & - r(s,a) - \gamma^{(L-1)/L} \EE_{s' \sim P(\cdot|s,a), \widetilde{a}_{1:L-1} \sim \widetilde{\pi}(\cdot|s')}\widetilde{\Tcal}^\pi f(s', \widetilde{a}_{1:L-1})|\\
        \le & \min_{f_1,\dots, f_{L} \in \Fcal} |f_1(s,a) - \widetilde{\Tcal}^\pi f_2(s,a)| \sum_{i=2}^{L} \gamma^{(i-1)/L}\EE_{s' \sim P(\cdot| s,a), \widetilde{a}_{1:i-1} \sim \widetilde{\pi}(\cdot|s'))}|f^i(s, \widetilde{a}_{1:i-1}) -\widetilde{\Tcal}^\pi((s, \widetilde{a}_{1:i-1})|\\
        \le & 0,
    \end{align*}
    where the last inequality follows from the assumption of token-level Bellman Complete.
\end{proof}
The next lemma assumes that the offline distribution $\nu$ is the same for both token-level MDP and utterance-level MDP, i.e. the token-level transitions are derived by splitting utterance-level transitions. We use $\widetilde{\nu}$ to denote the token-level offline distribution created in this way.
\begin{lemma}\label{lem: density_ratio}
    For any stationary policy $\pi$, and any offline distribution $\nu$, we have Density Ratio Coefficient for token-level $\widetilde{\Mcal}$ = Density Ratio Coefficient for utterance-level $\Mcal$.
    \begin{align*}
        C_{\nu, \pi} = C_{\widetilde{\nu}, \pi}
    \end{align*}
\end{lemma}
\begin{proof}
    The proof is constructed by noticing the fact that each token-level state consists of not only the utterance-level state but also all the tokens that have been generated in the current utterance:
    \begin{align*}
        &\max_{\widetilde{s} \in \widetilde{\Scal}, \widetilde{a} \in \widetilde{\Acal}}\frac{\widetilde{d}^\pi(\widetilde{s},\widetilde{a})}{\widetilde{\nu}(\widetilde{s}, \widetilde{a})}\\
        =& \max_{s \in \Scal, \widetilde{a}_{1:i-1} \in \widetilde{\Acal}, i \in [1,L]}\frac{\widetilde{d}^\pi(s,\widetilde{a}_{1:i-1}, \widetilde{a}_i)}{\widetilde{\nu}(s, \widetilde{a}_{1:i-1}, \widetilde{a}_i)}\\
        =& \max_{s \in \Scal, \widetilde{a}_{1:L} \in \widetilde{\Acal}}\frac{\widetilde{d}^\pi(s,\widetilde{a}_{1:L})}{\widetilde{\nu}(s, \widetilde{a}_{1:L})} \\
        =& \max_{s \in \Scal,a \in \Acal} \frac{d^\pi(s,a)}{\nu(s,a)},
    \end{align*}
    where the first equation follows by regrouping states and actions as each state in the token-level state space $\widetilde{s} \in \widetilde{\Scal} = \Scal \times \widetilde{\Acal}^L$, and the second equation holds because $\max_b \frac{p(a,b)}{q(a,b)} \ge 
    \frac{p(a)}{q(a)}$.
\end{proof}

\subsection{Proof of Main Theorem}
Now we are ready to analyze the sample complexity of Fitted Policy Evaluation in Algoritm \ref{alg:FPE}, the proof of our main theorem makes use of several technical lemmas from Section \ref{sec:technical_lemmas}.
\begin{lemma}[Guarantees of FPE] \label{lem:guarantees_fpe}
    For the algorithm described in Algorithm \ref{alg:FPE} with $K$ independent datasets \(\Dcal_{1:K} = \{(s,a, r, s')\}\) such that the effective number of samples $N=MK$, if the function class satisfies $\max_{f \in \Fcal, s \in \Scal, a \in \Acal} |f(s,a)| < R$, then with probability at least $1 - \delta$ the output value function satisfies:
    \begin{align*}
        \EE_{s,a \sim d^\pi}[(\bar{f}(s,a) - \Tcal^\pi \bar{f} (s,a))^2] \le C_{\nu, \pi} \left(\frac{4(R+1)^2}{K} + \frac{256 R^2\log(2|\Fcal|K/\delta)}{M} \right).
    \end{align*}
\end{lemma}
\begin{proof}
    By applying Lemma \ref{lem:sq_loss_generalization}, to each iteration $1, \dots, K$, and apply a union bound over all iterations, we have that:
    \begin{align*}
        \forall k, M\EE_{s,a \sim \nu}[(f_k(s,a) - \Tcal^\pi f_{k-1}(s,a))^2] &\le 3M\inf_{f \in \Fcal} \EE_{s,a \sim \nu}[(f(s,a) - \Tcal^\pi f_{k-1}(s,a))^2] + 256 R^2\log(2|\Fcal|K/\delta)\\
        \forall k, \EE_{s,a \sim \nu}[(f_k(s,a) - \Tcal^\pi f_{k-1}(s,a))^2] &\le \frac{256 R^2\log(2|\Fcal|K/\delta)}{M}, \numberthis \label{line:stat_error}\\
    \end{align*}
    where the second line holds by Assumption \ref{def:inherent_BE_error}. To combine the guarantees that we have from each round to the guarantee of the returned value function $\bar{f}$:
    \begin{align*}
        &\EE_{s,a \sim \nu}[(\bar{f}(s,a) - \Tcal^\pi \bar{f}(s,a))^2]\\
         =& \EE_{s,a \sim \nu}\left[\left(\frac{1}{K}(f_1(s,a) - \Tcal^\pi f_K(s,a) + \sum_{k=2}^K f_k(s,a) - \Tcal^\pi f_{k-1}(s,a)) \right)^2\right] \\
         \le& \frac{1}{K} \EE_{s,a \sim \nu} \left[(f_1(s,a) - \Tcal^\pi f_K(s,a))^2 + \sum_{k=2}^K (f_k(s,a) - \Tcal^\pi f_{k-1}(s,a))^2 \right]\\
         \le& \frac{1}{K} \left[4(R+1)^2 + (K-1)\frac{256 R^2\log(2|\Fcal|K/\delta)}{M} \right]\\
         \le& \frac{4(R+1)^2}{K} +  \frac{256 R^2\log(2|\Fcal|K/\delta)}{M}, \numberthis
    \end{align*}
    where the first inequality follows by Jensen's Inequality and the second inequality follows by plugging in Line \ref{line:stat_error}. Then, we can plug in Assumption \ref{def:offline_coverage} to conclude the proof.
    \begin{align*}
        \EE_{s,a \sim d^\pi}[(\bar{f}(s,a) - \Tcal^\pi \bar{f}(s,a))^2] \le & C_{\nu, \pi}\EE_{s,a \sim \nu}[(\bar{f}(s,a) - \Tcal^\pi \bar{f}(s,a))^2]\\
        \le&~  C_{\nu, \pi}(\frac{4(R+1)^2}{K} + \frac{256 R^2\log(2|\Fcal|K/\delta)}{M}). \\
        \le&~  64C_{\nu, \pi}R(R+1)\sqrt{\frac{\log(2|\Fcal|K/\delta)}{KM}}\\
        =&~  64C_{\nu, \pi}R(R+1)\sqrt{\frac{\log(2|\Fcal|K/\delta)}{N}}\\
        :=&~  \epsilon_{stat} \numberthis \label{line:eps_stat}
    \end{align*}
\end{proof}

\setcounter{theorem}{0}
\begin{theorem}[Main Theorem]
    For an utterance-level MDP with discount factor $\gamma^L$, where $L$ is the maximum length of each utterance, suppose utterance-level Assumption \ref{def:inherent_BE_error} and \ref{def:offline_coverage} holds, let $f$ be the final Q function returned by Fitted Policy Evaluation formalized in Algorithm \ref{alg:FPE} at the utterance level, yields a suboptimality gap of
    \begin{align*}
        &\EE_{s,a \sim d^\pi}[((\bar{f}(s,a) - \EE_{a' \sim \pi(\cdot|s)}[\bar{f}(s,a)]) - A^\pi(s,a))^2 ]\\
        \le& \frac{4}{1-\gamma}(\epsilon_{stat} + 2(R+1)\sqrt{\epsilon_{stat}})\\
        \le& \frac{4}{\gamma L(1-\gamma^{1/L})}(\epsilon_{stat} + 2(\frac{1}{1-\gamma}+1)\sqrt{\epsilon_{stat}}).
    \end{align*}
    For an equivalent token-level MDP with discount factor $\gamma^{1/L}$, suppose token-level Assumption \ref{def:inherent_BE_error} and \ref{def:offline_coverage} holds, let $f$ be the final Q function returned by Fitted Policy Evaluation formalized in Algorithm \ref{alg:FPE} at the token level, yields a suboptimality gap of
    \begin{align*}
        &\EE_{s,a \sim d^\pi}[((\bar{f}(s,a) - \EE_{a' \sim \pi(\cdot|s)}[\bar{f}(s,a)]) - \widetilde{A}^\pi(s,a))^2 ]\\
        \le& \frac{4}{(1-\gamma^{1/L})L^{1/2}}(\epsilon_{stat} + 2(\frac{1}{1-\gamma}+1)\sqrt{\epsilon_{stat}}L^{1/4}),\\
    \end{align*}
    where $\epsilon_{stat}$ is the statistical error defined in Line \ref{line:eps_stat} proportional to $N^{-1/2}$ the number of utterance-level transitions. This error term is defined the same for both utterance-level MDP and token-level MDP.
\end{theorem}
\begin{proof}
    First, we start with analyzing utterance-level FPE with effective number of samples, we start by bounding the errors of the Q functions:
    \begin{align*}
        &\EE_{s, a \sim d^\pi}(\bar{f}(s,a) - Q^\pi(s,a))^2 \\
        =& \EE_{s, a \sim d^\pi}(\bar{f}(s,a) - \Tcal^\pi\bar{f}(s,a) + \Tcal^\pi\bar{f}(s,a) - Q^\pi(s,a))^2\\
        =& \EE_{s, a \sim d^\pi} [(\bar{f}(s,a) - \Tcal^\pi\bar{f}(s,a))^2 + 2(\bar{f}(s,a) - \Tcal^\pi\bar{f}(s,a))(\Tcal^\pi\bar{f}(s,a) - Q^\pi(s,a)) + (\Tcal^\pi\bar{f}(s,a) - Q^\pi(s,a))^2]\\
        \le&  \EE_{s, a \sim d^\pi} [(\bar{f}(s,a) - \Tcal^\pi\bar{f}(s,a))^2] + 2\sqrt{ \EE_{s, a \sim d^\pi} (\bar{f}(s,a) - \Tcal^\pi\bar{f}(s,a))^2(\Tcal^\pi\bar{f}(s,a) - Q^\pi(s,a))^2} \\
        &+ \indent \indent \EE_{s, a \sim d^\pi}[(\Tcal^\pi\bar{f}(s,a) - Q^\pi(s,a))^2]\\
        \le&  \EE_{s, a \sim d^\pi} [(\bar{f}(s,a) - \Tcal^\pi\bar{f}(s,a))^2] + 2(R+1)\sqrt{ \EE_{s, a \sim d^\pi} (\bar{f}(s,a) - \Tcal^\pi\bar{f}(s,a))^2} + \EE_{s, a \sim d^\pi}[(\Tcal^\pi\bar{f}(s,a) - Q^\pi(s,a))^2]\\
        \le& \EE_{s, a \sim d^\pi} [(\bar{f}(s,a) - \Tcal^\pi\bar{f}(s,a))^2] + 2(R+1)\sqrt{ \EE_{s, a \sim d^\pi} (\bar{f}(s,a) - \Tcal^\pi\bar{f}(s,a))^2} \\
        &+ \gamma^2 \EE_{s, a \sim d^\pi, s' \sim P(\cdot|s,a), a' \sim \pi(\cdot|s') }[(\bar{f}(s',a') - Q^\pi(s',a'))^2]\\
        \le& \EE_{s, a \sim d^\pi} [(\bar{f}(s,a) - \Tcal^\pi\bar{f}(s,a))^2] + 2(R+1)\sqrt{ \EE_{s, a \sim d^\pi} (\bar{f}(s,a) - \Tcal^\pi\bar{f}(s,a))^2} + \gamma \EE_{s, a \sim d^\pi}(\bar{f}(s,a) - Q^\pi(s,a))^2,
    \end{align*}
    where the first inequality follows from Jensen's inequality, and the second inequality follows from \mbox{$\max_{s,a}\max\{|\Tcal^\pi\bar{f}(s,a)|, |Q^\pi(s,a)|\} \le R + 1$}, the third inequality follows again from Jensen's inequality,and the fourth inequality follows from Lemma \ref{lem:distribution_step_back}. Finally, we get:
    \begin{align*}
        &\EE_{s, a \sim d^\pi}(\bar{f}(s,a) - Q^\pi(s,a))^2\\
        \le& \frac{1}{1-\gamma}(\EE_{s, a \sim d^\pi} [(\bar{f}(s,a) - \Tcal^\pi\bar{f}(s,a))^2] + 2(R+1)\sqrt{ \EE_{s, a \sim d^\pi} (\bar{f}(s,a) - \Tcal^\pi\bar{f}(s,a))^2})\\
        =&  \frac{1}{1-\gamma}(\epsilon_{stat} + 2(R+1)\sqrt{\epsilon_{stat}}),
    \end{align*}
    where $\epsilon_{stat}$ is defined in Line \ref{line:eps_stat}. Finally, we can directly bound the errors for the advantages:
    \begin{align*}
        &\EE_{s,a \sim d^\pi}[((\bar{f}(s,a) - \EE_{a' \sim \pi(\cdot|s)}\bar{f}(s,a)) - A^\pi(s,a))^2 ]\\
        =& \EE_{s,a \sim d^\pi}[((\bar{f}(s,a) - \EE_{a' \sim \pi(\cdot|s)}\bar{f}(s,a)) - (Q^\pi(s,a) - \EE_{a' \sim \pi(\cdot|s)Q^\pi(s,a)}))^2 ]\\
        \le& 2 \EE_{s,a \sim d^\pi}[ (\bar{f}(s,a) - Q^\pi(s,a))^2] + 2\EE_{s \sim d^\pi}[(\EE_{a' \sim \pi(\cdot|s)}[\bar{f}(s,a') - Q^\pi(s,a')])^2]\\
        \le& 4 \EE_{s,a \sim d^\pi}[ (\bar{f}(s,a) - Q^\pi(s,a))^2]\\
        \le& \frac{4}{1-\gamma}(\epsilon_{stat} + 2(R+1)\sqrt{\epsilon_{stat}}),\\
        \le& \frac{4}{(1-\gamma^{1/L})\gamma L}(\epsilon_{stat} + 2(R+1)\sqrt{\epsilon_{stat}})\\
        =& \frac{4}{(1-\gamma^{1/L})\gamma L}(\epsilon_{stat} + 2(\frac{1}{1-\gamma}+1)\sqrt{\epsilon_{stat}})\\,
    \end{align*}
    where the second inequality follows from Jensen's inequality, and the fourth inequality follows from Lemma \ref{lem:gamma}. A similar analysis can be done with token-level FPE, noticing that the effective number of samples for token-level FPE is $NL$ because each utterance transition can be splitted into $L$ token transitions:
    \begin{align*}
        \EE_{s,a \sim \widetilde{d}^\pi}[((\bar{f}(s,a) - \EE_{a' \sim \pi(\cdot|s)}\bar{f}(s,a)) - \widetilde{A}^\pi(s,a))^2 ] \le& \frac{4}{1-\gamma}(\epsilon_{stat}L^{-1/2} + 2(R+1)\sqrt{\epsilon_{stat}}L^{-1/4}).\\
        \le& \frac{4}{(1-\gamma)L^{1/2}}(\epsilon_{stat} + 2(\frac{1}{1-\gamma}+1)\sqrt{\epsilon_{stat}}L^{1/4}).
    \end{align*}
\end{proof}

\subsection{Technical Lemmas} \label{sec:technical_lemmas}
First, we would like to prove an interesting lemma that allows us to compare the different discount factors for token level $\gamma^{1/L}$ and utterance level $\gamma$.
\begin{lemma}\label{lem:gamma}
    With the discount factor $\gamma \in (0,1)$, and $L$ being a positive number, we have that:
    \begin{align*}
        \frac{1}{1- \gamma} =& \frac{1}{(1 - \gamma^{1/L})(1 + \gamma^{1/L} + \gamma^{2/L} + \dots + \gamma^{(L-1)/L})}\\
        \le& \frac{1}{L\gamma(1-\gamma^{1/L})}
    \end{align*}
\end{lemma}
Below are some common technical lemmas useful for reinforcement learning.
\begin{lemma}[{\citep[Lemma 2]{zhou2023offline}}]  \label{lem:distribution_step_back}
For any policy $\pi$, and non-negative function $g(s,a)$, we have: 
\begin{align*}
\EE_{\bar s, \bar a \sim d^\pi}\EE_{s \sim P(\cdot | \bar s, \bar a), a\sim \pi(a | s)} \brk*{ g(s,a) } \leq  \frac{1}{\gamma} \EE_{s,a\sim d^\pi} \brk*{g(s,a)}. 
\end{align*}
where \(\mu_0\) denotes the initial state distribution (which is the same for all policies \(\pi\)). 
\end{lemma} 
\begin{proof} 
Recall that \(\lim_{h \rightarrow \infty} \gamma^h = 0\). We start by noting that: 
\begin{align*} 
d^{\pi}(s,a) & = (1-\gamma)  ( \mu_0(s, a) +  \gamma d^{\pi}_1(s,a) + \gamma^2 d^\pi_2(s,a) +\dots )  \numberthis \label{eq:step_back1} \\ 
& \geq \gamma (1-\gamma)  \left(  \sum_{\bar s, \bar a} \mu_0(\bar s, \bar a) P(s | \bar s, \bar a) \pi(a| s)  + \gamma  \sum_{\bar s, \bar a} d^\pi_1(\bar s, \bar a) P(s | \bar s, \bar a) \pi(a| s)      + \dots \right) \\ 
& =  \gamma (1-\gamma) \sum_{\bar s, \bar a}\left(  \mu_0(\bar s, \bar a) + \gamma d^\pi_1(\bar s, \bar a) + \dots     \right) P(s | \bar s, \bar a) \pi(a | s)   \\  
& =  \gamma \sum_{\bar s, \bar a} d^\pi(\bar s, \bar a) P(s | \bar s, \bar a) \pi(a | s)   \numberthis \label{eq:step_back2}  \\  
& = \gamma \EE_{\bar s, \bar a \sim d^\pi} \brk*{P(s | \bar s, \bar a) \pi(a | s) }, 
\end{align*}  
where Line \ref{eq:step_back2} follows by plugging in the relation in Line \ref{eq:step_back1}  for \(\bar s, \bar a\). The above implies that for any function $g \geq 0$, 
\begin{align*} 
\sum_{s, a} d^\pi(s, a) g(s,a) \geq \sum_{s,a } \gamma \EE_{\bar s, \bar a \sim d^\pi} \brk{P(s | \bar s, \bar a) \pi(a | s) g(s,a)}, 
\end{align*} which implies that 
\begin{align*}
\EE_{\bar s, \bar a \sim d^\pi}\EE_{s \sim P(\cdot | \bar s, \bar a), a\sim \pi(a | s)} \brk*{ g(s,a) } \leq  \frac{1}{\gamma} \EE_{s,a\sim d^\pi} \brk*{g(s,a)}. 
\end{align*}
\end{proof}

\begin{lemma}[{Least squares generalization bound, \citep[Lemma 3]{song2023hybrid}}] 
\label{lem:sq_loss_generalization} 
Let \(R > 0\), \(\delta \in (0, 1)\), and consider a sequential function estimation setting with an instance space $\Xcal$ and target space $\Ycal$. Let \(\Hcal: \Xcal \mapsto [-R, R]\) be a class of real valued functions. Let \(\Dcal = \crl*{(x_1, y_1), \dots, (x_T, y_T)}\) be a dataset of \(T\) points where $x_t \sim \rho_t := \rho_t(x_{1:t-1},y_{1:t-1})$, and \(y_t\) is sampled via the conditional probability $p_t( x_t)$ (which could be adversarially chosen). Additionally, suppose that \(\max_t \abs{y_t} \leq R\) and \(\max_{h} \max_{x} \abs*{h(x)} \leq R\). Then, the least square solution \(\widehat{h} \leftarrow \argmin_{h \in \Hcal} \sum_{t=1}^T \prn*{h(x_t) - y_t}^2\) satisfies  
\begin{align*} 
 \sum_{t=1}^T \EE_{x \sim \rho_t, y \sim p_t(x)} \brk*{\prn{\widehat{h}(x) - y}^2} &\leq 3\inf_{h \in \Hcal}  \sum_{t=1}^T \EE_{x \sim \rho_t, y \sim p_t(x)} \brk*{\prn{h(x) - y}^2}  + 256 R^2 \log(2 \abs{\Hcal}/\delta) 
\end{align*} 
with probability at least \(1 - \delta\). 
\end{lemma} 


\end{document}